\documentclass[12pt]{article}

\usepackage[margin=1in]{geometry}  
 
\usepackage{mathrsfs}
\usepackage{kpfonts,comment}

\usepackage{natbib,bm}   

\newcommand{\moronly}[1]{{}}

\usepackage{graphicx,amsmath,amsfonts,amscd,amssymb,bm,epsfig,epsf,url,color,algorithm,algorithmic,enumerate}
\usepackage{amsthm,color,xcolor,mathtools} 

\usepackage[utf8]{inputenc} 
\usepackage[T1]{fontenc}    
\usepackage{hyperref}       
\usepackage{url}            
\usepackage{booktabs}       
\usepackage{amsfonts}       
\usepackage{nicefrac}       
\usepackage{xcolor}
\usepackage{array}

\newcommand{\kz}[1]{\textcolor{red}{[Kaiqing: #1]}}
\newcommand{\spp}[1]{\textcolor{blue}{[Sarath: #1]}}

\newcounter{assumption}
\newcounter{subAssumption}
\newcounter{assumptionPrime}
\newcounter{subAssumptionPrime}
\renewcommand\theassumptionPrime{\arabic{assumptionPrime}'}
\renewcommand\thesubAssumption{\arabic{assumption}-\roman{subAssumption}}
\renewcommand\thesubAssumptionPrime{\arabic{assumptionPrime}'-\roman{subAssumptionPrime}}

\newenvironment{assumption*}{\setcounter{assumptionPrime}{\theassumption-1}\refstepcounter{assumptionPrime}\par
   \noindent  \textbf{Assumption \theassumptionPrime.} \em \rmfamily}

\newenvironment{assumptions*}{\setcounter{assumptionPrime}{\theassumption-1}\refstepcounter{assumptionPrime}\setcounter{subAssumptionPrime}{0}\par
\rmfamily}

\newenvironment{subAssumption*}{\refstepcounter{subAssumptionPrime}\par
   \noindent  \textbf{Assumption \thesubAssumptionPrime.} \em \rmfamily}
   
\newcommand\Dc{\ensuremath{\mathcal{D}}}

\newtheorem{Lemma}{Lemma}
\newtheorem{Proposition}{Proposition}
\newtheorem{Theorem}{Theorem}
\newtheorem{Definition}{Definition}
\newtheorem{Corollary}{Corollary}

\newtheorem{Remark}{Remark}
\newtheorem{Assumption}{Assumption}

\newcommand\EE{\mathbb{E}}

\usepackage{tikz}

\newcolumntype{P}[1]{>{\centering\arraybackslash}p{#1}}

\usepackage{algorithmic,algorithm}
\usepackage{array}
\usepackage{enumerate}

\def\x{{\mathbf x}}

\renewcommand\footnotemark{} 

\title{\Large Offline Reinforcement Learning via Linear-Programming  with Error-Bound Induced  Constraints}
\author{\normalsize Asuman Ozdaglar$^\dag$ \and \normalsize Sarath Pattathil$^\dag$ \and \normalsize Jiawei Zhang$^\dag$ \thanks{
$^\dag$Massachusetts Institute of Technology. $^\ddagger$University of Maryland, College Park. \{asuman, sarathp, jwzhang\}@mit.edu; kaiqing@umd.edu.} \and \normalsize Kaiqing Zhang$^\ddagger$}
\date{} 

\begin{document}
	
	\maketitle
	
	\begin{abstract} 
Offline reinforcement learning (RL) aims to find an optimal policy for Markov decision processes (MDPs),  
using a pre-collected dataset, without further interactions with the environment. 
In this work, we revisit the linear programming (LP) reformulation of Markov decision processes for offline RL, with the goal of developing algorithms with optimal  $O(1/\sqrt{n})$ sample complexity, where $n$ is the sample size, under partial data coverage and general function approximation, and with favorable computational tractability. To this end, we derive new \emph{error bounds} for both the dual and primal-dual formulations of the LP, and incorporate them properly as \emph{constraints} in the LP reformulation. We then show that under a completeness-type assumption,   
$O(1/\sqrt{n})$ sample complexity can be achieved under standard single-policy coverage  assumption, when one properly  \emph{relaxes} the occupancy validity constraint in the LP. This framework can readily handle both infinite-horizon discounted and average-reward MDPs, in  both general function approximation and tabular cases. The instantiation to the tabular case achieves either state-of-the-art or the first sample complexities of offline RL in these settings. 
To further remove any completeness-type assumption, we then introduce a proper \emph{lower-bound constraint} in the LP, and a variant of the standard single-policy coverage assumption. Such an  algorithm leads to a  $O(1/\sqrt{n})$ sample complexity with dependence on the \emph{value-function gap},  with only realizability assumptions. 
Our properly constrained LP-framework    advances the existing results in several aspects, in relaxing certain assumptions and achieving the optimal $O(1/\sqrt{n})$ sample complexity, with simple analyses.  
We hope our results bring new insights into the use of LP formulations and the equivalent  primal-dual minimax  optimization for offline RL,  through the error-bound induced constraints. 
	\end{abstract}


\section{Introduction}\label{sec:intro}


Recent years have witnessed tremendous empirical successes of reinforcement learning (RL) in many sequential-decision making problems \citep{mnih2015human,silver2016mastering,vinyals2017starcraft,levine2016end}. Key to these successes are two factors: 1) use of rich {\it function approximators}, e.g., deep neural networks; 2) access to excessively {\it large interaction data} with the environment.  Most  successful examples above are extremely data-hungry. In some cases, the  interaction data can be easily obtained in an online fashion, due to the existence of powerful simulators such as game engines \citep{silver2016mastering,vinyals2017starcraft} and physics simulators \citep{todorov2012mujoco}. 

On the other hand, in many other domains of RL, such online interaction is impractical, either because data collection is  expensive and/or impractical, or the environment is simply difficult to simulate well. Many real world applications fall into this setting, including robotics and autonomous driving \citep{levine2018learning,maddern20171}, healthcare \citep{tseng2017deep}, and recommender systems \citep{swaminathan2017off}. Moreover, even in the cases where online interaction is available, one might still want to utilize previously collected data, as effective generalization requires {\it large}  datasets  \citep{levine2020offline}. \emph{Offline RL} has thus provided a promising framework when one really targets deploying RL in the real-world.

However, in practice, offline RL  is known to suffer from the {\it training instability} issue due to the use of function approximation, e.g., neural networks, and the {\it distribution shift}   issue due to the mismatch between the offline data distribution and the targeted (optimal) policy distribution  
\citep{fujimoto19off,kumar20conservative}. As a result sample-efficiency guarentees for offline RL  with function approximation  usually relies on strong assumptions on both the {\it function classes} and the {\it dataset}. In particular, many earlier results \citep{munos2008finite,scherrer2014approximate,chen2019information,zhang2021finite} require the function classes to be {\it Bellman-complete}, i.e., the value function class is {\it closed} under the Bellman operator, and the dataset to have {\it full coverage}, i.e., the data covers the state distributions induced by {\it all} policies.  Both assumptions are strong: the former is {\it non-monotone} in the function class, i.e.,  the assumption can be violated when a richer function class is used, and is much stronger than the common assumption of {\it realizability} (i.e., the optimal solution lies in the function class) in statistical learning theory; the latter essentially  requires the offline data to  cover all the possible state-action pairs, which is violated in most real-world applications.

Significant progress has been made lately to relax these assumptions.   For example, \cite{liu2020provably,jin2020ispessimism,rashidinejad2021bridging,xie2021bellman,uehara2021pessimistic} have shown that using the \emph{pessimistic} mechanism that chooses the worst-cast value function or model in the uncertainty set during learning, the full coverage assumption can be relaxed to only a {\it single-policy} (i.e., an optimal policy) coverage assumption. Nonetheless, the results all rely on completeness-type (which includes the tabular setting) or even stronger assumptions, and some of the algorithms are not computationally tractable \citep{xie2021bellman,uehara2021pessimistic}. On the other hand, some works require only realizability, but with additionally either stronger (than all-policy coverage) assumptions on data coverage \citep{xie2021batch}, or the  uniqueness of the optimal policy \citep{chen2022offline}, and can be computationally intractable. 

More recently, \cite{zhan2022offline} has successfully  relaxed  {\it both} the full data coverage and the completeness assumptions, through the 
use of  the linear-programming (LP) reformulation of   Markov decision processes (MDPs). 
The LP framework not only significantly weakens the assumptions, but also better enables computationally tractable algorithms.  However, the algorithms and analyses in \cite{zhan2022offline}  strongly depend on a {\it regularized} version of the LP formulation, which calls for stronger assumptions than {\it single-policy} coverage, and leads to statistically suboptimal rates (i.e., $O(1/{n^{1/6}})$ where $n$ is the size of the dataset). In this paper, we  revisit and further investigate the power of the LP-framework for offline RL with performance guarantees, through the lens of error-bound-based analyses from  optimization theory \citep{luo1993error,pang1997error}. Our results  advance the existing ones in several aspects with new techniques  and insights, as well as simple analyses. 


\paragraph{Contributions.} We propose LP-based offline RL algorithms with  optimal (in terms of sample size $n$) $O(1/\sqrt{n})$ sample complexities,  
under partial data coverage and general function approximation, and without any behavioral regularization  \citep{zhan2022offline}. In particular, first, we obtain the $O(1/\sqrt{n})$ optimal rate under the standard single-policy concentrability (SPC) assumption \citep{rashidinejad2021bridging}, with  some completeness-type assumption on the function class, for infinite-horizon discounted MDPs. Second, our  result leads to the near-optimal rate of $O(\sqrt{|S|}/((1-\gamma)\sqrt{n}))$ when reducing  to the tabular case,  improving even the state-of-the-art tabular-case result \citep{rashidinejad2021bridging}. Most results for the general function  approximation case yield a loose bound from such a reduction. 
Third, the framework can also be readily generalized to handle \emph{average-reward} MDPs, attaining one of the first sample complexity results of offline RL in this setting.  
Fourth, with only the realizability assumption,  we obtain $O(1/({\texttt{Gap}} \cdot\sqrt{n}))$ rate under a partial data coverage assumption that is a variant of standard SPC, where {\texttt{Gap}} denotes the minimal difference between the values of the best   action and the second-best one among all states. Note that our algorithms inherit the  favorable computational tractability as other LP-based offline RL algorithms \citep{zhan2022offline,rashidinejad2022optimal}, with basic and neat analysis. 
Key to our algorithms is to add proper  \emph{constraints} in the LP, inspired by a few novel \emph{error bounds} of the LP-based offline formulation (see detailed discussions  next). Finally, we generalize our framework to several other scenarios in a unified fashion, including those {with sparse state-visitation by the optimal policy}, with unknown behavior policy and no-coverage of the optimal policy, 
  and contextual bandits (as a special case of MDPs) without the aforementioned completeness assumption.  



\paragraph{Our key techniques:  LP with Error-bound induced constraints.} The key idea to our approaches is to study the LP reformulations of the underlying MDPs under additional {\it constraints} induced by certain \emph{error bounds}.  To this end, 
we focus on the dual problem of a variant of the standard LP reformulation, based on the marginal importance sampling (MIS) framework  \citep{nachum2019algaedice,lee2021optidice}, where the dual variable   corresponds to the ratio between the state-action occupancy measure and the offline data distribution (also referred to as the {\it density ratio}).  Error-bound analysis is a powerful tool in optimization theory, which is generally defined as estimating  the suboptimality  of a   solution  by   some \emph{residual functions}, defined as the \emph{amount of violation of certain controllable optimality conditions}. 

Our first set of results relies on a key {\it error bound} lemma that relates the value function suboptimality with the $\ell_1$-norm \emph{violation of the validity constraint} on the \emph{occupancy measure} in the LP (see Lemma \ref{lemma:infeasible}). This lemma leads to a {\it constrained}  dual formulation  without the need of behavior regularization as in \cite{zhan2022offline}. Using function approximation for the density ratio and the {\it sign function} of the occupancy validity  constraint (see Definition \ref{def:x_sign}), this formulation organically allows us to obtain $O(1/\sqrt{n})$ sample complexity under the realizability of density ratio function class, and certain completeness assumption on the {sign function}, together with the  standard SPC assumption \citep{rashidinejad2021bridging,chen2022offline}.  More importantly, such a framework can readily handle multiple extended settings, and can also 
be readily 
instantiated to the tabular MDP case, for \emph{both} discounted and average-reward settings, achieving either state-of-the-art or the first sample complexities in the  literature. 

To remove any completeness assumption, we then  consider the \emph{minimax reformulation} of the dual LP, which  dualizes the occupancy measure validity constraints.  We then develop a new {\it error bound} that relates the value function suboptimality with the \emph{primal gap} \citep{ozdaglar2022good} of the minimax reformulation (see Lemma \ref{J_Delta}). To stabilize the normalization step in generating the policy from the LP solution (see Eq.  \eqref{equ:def_pi_w}), we introduce an additional {\it lower-bound constraint}   on the density ratio, which does not lose optimality if the initial state distribution coincides with the offline data distribution.  
Under this new formulation, we establish gap-dependent $O(1/\sqrt{n})$ sample complexity with \emph{only}  realizability assumptions on the value function and density ration, and a slightly stronger SPC assumption, which assumes  that certain optimal policy covered by the behavior policy is also covered by the offline data. 




\subsection{Related Work}
\label{app:sec_1_1}

We 
categorize  the literature  based on the assumptions on data and function class.

\paragraph{Data coverage assumptions.} Early theoretical works on offline RL usually require  the {\it all-policy concentrability} assumption, i.e., the offline data has to be exploratory enough to cover the  state distributions induced by {\it all} policies \citep{munos2008finite,scherrer2014approximate,chen2019information}. We refer to this assumption as the {\it full data  coverage} assumption. Slightly weaker variant that assumes some weighted version of the all-policy concentrability coefficient is bounded has also been investigated \citep{xie2021batch,uehara2020minimax}. More recently, significant progress has been made to relax full coverage assumption to partial coverage ones. \cite{jin2020ispessimism,rashidinejad2021bridging,li2022settling} developed pessimistic value iteration based algorithms for tabular or linear MDPs, under the single-policy concentrability assumption on data coverage. When general function approximation is used, some variants of the SPC assumption were proposed to account for partial data coverage \citep{uehara2021pessimistic,xie2021bellman,cheng2022adversarially}. However, these algorithms are either computationally intractable \citep{uehara2021pessimistic,xie2021bellman}, or statistically suboptimal \citep{cheng2022adversarially}. 
Other recent works that require only partial data coverage are \cite{zhan2022offline,chen2022offline} which will be discussed next. Finally, after acceptance of our conference version \cite{ozdaglar2023revisiting}, \cite{zhu2024importance} also achieved  $O(1/\sqrt{n})$ sample complexity under partial data coverage and the actor-critic paradigm.  
Notably, the rate holds in competing with the best covered policy (as also in our \S\ref{sec:extension_best_covered_policy}), under a weaker notion of $\ell_2$-SPC.   

\paragraph{Function class assumptions.} 
One common assumption on function class is the  Bellman-completeness  on value functions \citep{munos2008finite,scherrer2014approximate,chen2019information,xie2021bellman,cheng2022adversarially,zhu2024importance}, which requires the value function class to be closed under the Bellman operator. By definition, such an assumption is automatically satisfied for the tabular and linear MDP cases mentioned above \citep{jin2020ispessimism,rashidinejad2021bridging,li2022settling}, and is implied when realizability of the MDP model \citep{uehara2021pessimistic} is assumed, see \cite{chen2019information}.  This strong assumption has been recently relaxed to only {\it realizability}, i.e., the function class only needs to contain (approximately) the target function of interest (e.g., optimal value function) \citep{xie2021batch}. However, \cite{xie2021batch} relies on data coverage assumption that is even stronger than all-policy concentrability. In fact, there have been hardness results \citep{wang2020statistical,amortila2020variant,zanette2021exponential,foster2021offline} showing that even with good data coverage, realizability-only assumption on the value function is not sufficient for sample-efficient offline RL. This motivated the use of   function approximation for density ratio (in addition to value function), as in \cite{nachum2019algaedice,zhan2022offline,chen2022offline,jiang2020minimax} and our work. In particular, \cite{zhan2022offline,chen2022offline} are the most related recent works that assume only realizability, on both value function and density ratio, and partial data coverage. However, they are either statistically suboptimal \citep{zhan2022offline} or computationally intractable \citep{chen2022offline}. Moreover, \cite{zhan2022offline} additionally requires the data coverage of the {\it regularized} problem; and \cite{chen2022offline} additionally requires  that the greedy optimal action is {\it unique} for all states.

\paragraph{Independent works  \cite{rashidinejad2022optimal,gabbianelli2024offline}.} While preparing our paper, we came across independent works   \cite{rashidinejad2022optimal,gabbianelli2024offline} that are closely related. Concurrent to our first arxiv version, \cite{rashidinejad2022optimal} also considers the general function approximation setting and discounted MDPs, and also obtained the optimal $O(1/\sqrt{n})$ rate 
 via an LP framework, 
 without behavioral regularization. Note that \cite{rashidinejad2022optimal} requires completeness-type assumptions throughout, which can be viewed as mirroring the first part of our results (c.f. \S\ref{sec:func_1}), while we also have the realizability-only results under a slightly  different  data coverage  assumption (c.f. \S\ref{sec:second_FA}), and other  extensions (c.f. \S\ref{sec:armdp} and \S\ref{sec:extensions}). Even comparing our \S\ref{sec:func_1} and  \cite{rashidinejad2022optimal}, there are several key differences: 
First, \cite{rashidinejad2022optimal} is based on an augmented Lagrangian method (ALM), while we propose to solve the   optimization with constraints directly. Second, the function classes being used, and the corresponding completeness and realizability assumptions are different (see \S\ref{sec:func_1} for more detailed discussions).  Third, with a different and rather simple analysis,  our results have better dependence on $(1-\gamma)$, and even improves the state-of-the-art result when specializing to the tabular case \citep{rashidinejad2021bridging}. Finally, and interestingly, we also note that both works have noticed the importance of occupancy validity constraints, and our constrained formulation in \S\ref{sec:func_1} mirrors the role of ALM in \cite{rashidinejad2022optimal}, to \emph{enforce} such validity constraints. {In addition, our framework also readily covers several other settings in a unified way, including average-reward MDPs (c.f. \S\ref{sec:armdp}), settings with unknown behavior policies or only coverage of suboptimal policies (c.f. \S\ref{sec:extensions}), and contextual bandits without completeness (c.f. \S\ref{sec:extensions}).}   

After the acceptance of our conference version \cite{ozdaglar2023revisiting} and workshop version \cite{ozdaglarpower}, \cite{gabbianelli2024offline} also proposes to exploit  the \emph{LP framework} for offline RL in \emph{linear} MDPs specifically. Similar to the present paper (and also our workshop version \cite{ozdaglarpower}), \cite{gabbianelli2024offline} can also handle average-reward (linear) MDPs. Besides the difference of  \emph{general}  versus  \emph{linear}  function approximation, there are several other key differences compared to the present paper: first, for the discounted-reward  setting,  
the sample complexity in \cite{gabbianelli2024offline} is $O(1/n^{1/4})$,  while ours is $O(1/\sqrt{n})$   with also a better dependence on $(1-\gamma)$; for the average-reward setting, their sample complexity remains $O(1/n^{1/4})$ under the bounded bias-function-span assumption, while ours remains  $O(1/\sqrt{n})$ under a stronger  bounded mixing-time assumption, making the results incomparable.   
Second, by exploiting the linear MDP structure, the primal-dual ascent-descent algorithm in \cite{gabbianelli2024offline} has a concrete computational complexity of $O(n)$,  while ours is efficient under a convex optimization oracle with general function approximation. 
Again, intriguingly, as in \cite{rashidinejad2022optimal} and our paper, one key in \cite{gabbianelli2024offline} is the proper \emph{relaxation} of the validity constraints by introducing an new auxiliary dual variable.  Finally, our paper also covers several extension settings including  unknown behavior policies 
(c.f. \S\ref{sec:extensions}) and contextual bandits without completeness (c.f. \S\ref{sec:extensions}). 


\subsection{Notation}
\label{app:sec_1_2}
For a positive integer $n$, we use $[n]$ to denote the set  $\{1,\cdots,n\}$. 
For a vector $v\in\mathbb{R}^d$, we use $\|v\|_p$ to denote its  $\ell_p$ norm (where $p\in[0,\infty]$, and if there is no subscript, $\|v\|$ denotes the $\ell_2$ norm. Note that $\|v\|_0$ denotes the number of non-zero elements in $v$. For a matrix $M\in\mathbb{R}^{m\times n}$, we use $\|M\|$ and $\|M\|_{F}$ to denote its $\ell_2$-induced norm and Frobenius norm, respectively, and use $M^\top$ to denote its transpose. For a set $S$, we use $|S|$ to denote its  cardinality, and $\Delta(S)$ to denote the probability distribution over $S$. For a function class $\mathcal{F}$, we use $|\mathcal{F}|$ to denote its cardinality if it is discrete, and its covering number if it is continuous.  We use $\EE$ to denote expectation.  For any matrix $M\in\mathbb{R}^{m\times n}$,  $M\geq 0$  and $M>0$ denote  the facts that each element of $M$ is non-negative and positive, respectively. We also use $\mathbb{R}^d_+$ to denote the $d$-dimensional real vector space with all elements being non-negative. We use $\mathrm{Diag}(M_1,\cdots,M_n)$ to denote the block diagonal matrix of proper dimension  whose diagonal blocks $M_1,\cdots,M_n$ have the same dimension. We follow the convention of $0/0=0$ throughout, unless otherwise noted.

\section{Background}\label{sec:background}

\subsection{Model and Setup}\label{sec:model_setup}

\paragraph{Discounted MDP.} 
Consider an  infinite-horizon discounted  MDP characterized by a  tuple $\langle S,A,P,R,\gamma,\mu_0\rangle$, where $S=\{s^1,\cdots,s^{|S|}\}$ and $A=\{a^1,\cdots,a^{|A|}\}$ denote  the state and action spaces of the agent,  $R:S\times A\to [0,1]$ is the reward function\footnote{Note that we stick to the case of deterministic reward for ease of presentation as in the literature \citep{xie2021bellman,zhan2022offline}. Our results can  be readily extended to the case of random rewards.}, $P:S\times A\to \Delta(S)$ denotes the transition kernel, $\gamma\in[0,1)$ denotes the discount  factor, and $\mu_0\in\Delta(S)$ denotes the initial state distribution. We assume $S$ and $A$ are finite (but potentially very large), in order to ease the presentation. However, note that our results later do not depend on the  cardinalities of $S$ and $A$ when function approximation is used.  Let $\pi:S\to \Delta(A)$ denote a Markov stationary policy of the agent, determining the  distribution over actions at each state. {Each $\pi$ leads to a \textit{discounted occupancy}   measure over the state-action spaces, defined by 
\vspace{6pt}
\begin{equation}\label{equ:def_theta_pi}
	\begin{array}{ll}
	\theta_{\pi,\mu_0}(s,a)=(1-\gamma)\sum_{t=0}^{\infty} \gamma^t\mathbb{P}_{\pi}(s_t=s, a_t=a;\mu_0),
\end{array}	 
\end{equation} 
\vspace{6pt}
where $\mathbb{P}_{\pi}(s_t=s,a_t=a;\mu_0)$ is the probability of the event of visiting the pair $(s,a)$ at time $t$ under the policy $\pi$, starting from $s_0\sim \mu_0(\cdot)$.

}
Correspondingly, with a slight abuse of notation, we use $\theta_{\pi,\mu_0}(s)=\sum_{a\in A}\theta_{\pi,\mu_0}(s,a)$ to denote the \textit{discounted occupancy   measure over states}. For notational convenience, we concatenate the state-action occupancy measure $\theta_{\pi, \mu_0} (s,a)$ in a vector $\theta_{\pi, \mu_0}$, defined as
\begin{align} \label{theta_pi}
&\theta_{\pi,\mu_0}=\big[\theta_{\pi,\mu_0}(s^1,a^1),\cdots, \theta_{\pi,\mu_0}(s^1,a^{|A|}),\cdots, \theta_{\pi,\mu_0}(s^{|S|},a^1), \cdots,\theta_{\pi,\mu_0}(s^{|S|},a^{|A|})\big]^{\top} \in \mathbb{R}_+^{|S| |A|}.
\end{align}

Given any policy $\pi$, one can then define the corresponding  state-action and state  value functions, $Q_\pi$ and $v_\pi$, as follows:
\begin{equation*}
Q_{\pi}(s,a)=  \EE_{s_{t+1}\sim P_{s_t,a_t}(\cdot),a_t\sim\pi(\cdot\mid s_t)} \left[ \sum_{t=0}^{\infty}\gamma^t r(s_t,a_t) \Bigg| s_0 = s, a_0 = a \right],
\end{equation*} 
where the trajectory is generated  following the policy $\pi$,
and $v_{\pi}(s)=\EE_{a\sim\pi(\cdot \mid  s)}[Q_{\pi}(s,a)].$
The overall goal is to find a policy $\pi^*$ that solves the following problem:
\begin{align}\label{equ:main_obj}
	\max_{\pi}~~~J_{\mu_0}(\pi):=(1-\gamma)\cdot\EE_{s\sim \mu_0}\big[v_{\pi}(s)\big],
\end{align}
where $J_{\mu_0}(\pi)$ denotes the {\it return} under $\pi$ and $\mu_0$, i.e., the $(1-\gamma)$-times expected value function under policy $\pi$ and  initial distribution $\mu_0$.  Note that $J_{\mu_0}(\pi)$ can also be equivalently written as $J_{\mu_0}(\pi)=r^\top\theta_{\pi, \mu_0}$, where
\begin{align}
r= [r(s^1,a^1), \cdots,r(s^1,a^{|A|}),\cdots, r(s^{|S|},a^{1}), \cdots, r(s^{|S|},a^{|A|})]^\top\in [0,1]^{|S||A|}.
\end{align}
{It is known that the optimal solution to the MDP is a Markov stationary policy.} For a general distribution $\rho\in\Delta(S)$, we use $\theta_{\pi,\rho}$ and $J_{\rho}(\pi)$ to denote the discounted occupancy measure and the average value function under policy $\pi$, but starting from another state  distribution  $\rho$. We sometimes just write $\theta_\pi$ and $J(\pi)$ for simplicity, when the initial distribution is clear from context. 
Note that the optimal policy $\pi^*$ may not be unique.  We define $v^*=v_{\pi^*}$ and $Q^*=Q_{\pi^*}$. For convenience, we sometimes denote  $m=|S||A|$.


\paragraph{Contextual bandit.}

Contextual bandit is the special case of reinforcement learning in discounted MDPs where $\gamma=0$. 
A motivating example is the learning in recommendation systems \citep{li2010contextual}. 
Consider a recommendation system that provides suggestions on items or actions based on the user's context.
At each round $t$, the user' context can be viewed as the state  $s_t$, and the suggested items can be viewed as actions $a_t$. 
The system aims to learn a policy that maximizes the expected  reward. 

\paragraph{Average-reward MDP.}
Another important model  of MDPs in RL is the average-reward one, where the 
goal is to maximize the expected average reward
\begin{equation}\label{equ:obj_avg_mdp}
J(\pi):=\lim_{T\rightarrow \infty}E \left[ \frac{1}{T}\sum_{t=0}^{T-1}r(s_t,a_t) \right],
\end{equation}
where $s_{t+1}\sim P_{s_t,a_t}(\cdot),a_t\sim\pi(\cdot\mid s_t)$ follows the trajectory generated by policy $\pi$. 
With a slight abuse of notation, we use $\theta_{\pi}$ to denote the \emph{stationary distribution} of the Markov process  $(s_t,a_t)_{t\geq 0}$ under policy $\pi$, which can be obtained (if exists) by $\theta_{\pi}(s,a)=\lim_{t\to\infty}\mathbb{P}_{\pi}(s_t=s,a_t=a;\mu_0)$. Note that in this case, the initial distribution $\mu_0$ does not affect the stationary distribution and we thus omit the subscript $\mu_0$. 
Similarly, we define $\theta_{\pi}(s)=\sum_{a\in A}\theta_{\pi}(s,a)$ to be the stationary distribution over state $s$. In this case, the objective $J(\pi)$ can also be equivalently re-written as $J(\pi)=r^\top\theta_{\pi}$. We also use $\pi^*$ to denote the optimal policy that maximizes $J(\pi)$. 

\paragraph{Offline RL.} Consider an offline RL problem, where one has collected a dataset  $\Dc$ containing $n$ samples. 
Suppose $\Dc=\{(s_i,a_i,s_i',r_i)\} _{i=1}^n$, where the independent and identically distributed  (i.i.d.)  samples $(s_i,a_i)\sim\mu(\cdot,\cdot)$. We denote $\mu(s)=\sum_a\mu(s,a)$, which implies that $s_i$ are drawn i.i.d. from the distribution $\mu(\cdot)$.  We denote the conditional distribution of $a$ given $s$ induced from $\mu$ as $\pi_\mu(a \mid  s)$ for $\mu(s)>0$, 
{i.e., $\pi_\mu(a \mid  s) = \mu(s,a)/\mu(s)$; and for $\mu(s)=0$, $\pi_\mu(\cdot\mid s)$ is ill-defined, and in order to be consistent with its later use, we define it as $\pi_\mu(a \mid  s)=0$ for all $a\in A$.}  
$\pi_\mu$ is usually  also referred to as the {\it behavior policy} if $\mu$ happens to correspond  to the occupancy measure or stationary distribution  of some   policy. 

In most parts of the  paper, we assume that the behavior policy $\pi_\mu(a \mid  s)$  is known, as in \cite{zhan2022offline,rashidinejad2022optimal}. 
{We provide extensions of our algorithms to the scenarios when the behavior policy is not known in Appendix \ref{sec:extensions}.}
Given a state-action pair $(s_i,a_i)$, we have $r_i=r(s_i,a_i)$ and $s'_i\sim P_{s_i,a_i}(\cdot)$. Moreover, let $n_{\Dc}(s,a)$ be the subset of the sample indices $\{1,\cdots,n\}$ that includes the indices of the samples in $\Dc$ that visit state-action pair in the sense of $(s_i,a_i)=(s,a)$. Similarly,  we use  $n_{\Dc}(s,a,s')$ and $n_\Dc(s)$ to denote the sets of indices of data samples in $\Dc$ such that $(s_i,a_i,s_i')=(s,a,s')$  and $s_i=s$, respectively. We define the empirical version of $\mu$, i.e., $\mu_{\Dc}$, as $\mu_{\Dc}(s,a)=n_{\Dc}(s,a)/n$. The goal of offline RL is to make use of the dataset $\Dc$ to learn a policy $\hat{\pi}$, such that the  {\it optimality gap} $J_{\mu_0}(\pi^*)-J_{\mu_0}(\hat{\pi})$ (in the discounted MDP case) and $J(\pi^*)-J(\hat{\pi})$ (in the average-reward MDP case), are small. 

\subsection{LP-based Reformulations}\label{sec:LP_reform}

\subsubsection{Discounted MDP}
It is known that for tabular MDPs, any optimal policy $\pi^*$ optimizes $J_{\rho}(\pi)$ starting from any  distribution $\rho\in\Delta(S)$ (including  the actual initial state distribution $\rho=\mu_0$ in the model in \S\ref{sec:model_setup}) \citep{puterman1994markov}, as it simultaneously maximizes $v_{\pi}(s)$ for all states $s \in S$.
	Moreover, the optimality condition of the MDP  when starting from any distribution {$\rho$} can also be written as the following linear program    \citep{puterman1994markov}: 
	\begin{equation}\label{P00}
	\begin{array}{ll} 
	\min_{v}~~~  (1-\gamma){\rho}^\top v \qquad\qquad 
	\mbox{s.t.~~~~} & \gamma P_{(s,a)}^\top v+r(s,a)\le v(s), \quad \forall s\in S,~a\in A,
	\end{array}
	\end{equation} 
	where $P_{(s,a)}=[P_{s,a}(s^1),\cdots, P_{s,a}(s^{|S|})]^\top\in\Delta(S)$ is the vector of state transition probabilities for the state-action pair $(s,a)$.
	Let $P=[P_{(s^1,a^1)},\cdots, P_{(s^1,a^{|A|})}, \cdots, P_{(s^{|S|},a^1)},\cdots,P_{(s^{|S|},a^{|A|})}]\in \mathbb{R}^{|S|\times m}$ and $\bm{1}_{|A|} = [1,1,\cdots,1]^\top\in \mathbb{R}^{|A|}$.
	
Note that we keep the initial-state distribution used in the LP \eqref{P00} to be $\rho$ (instead of $\mu_0$) for generality, which does not affect the solution in the tabular case. However, as we will specify in later sections, the choice of $\rho$ can make a difference in the function approximation setting, and may help address some challenging settings in offline RL.  Interestingly, such a distinction  has already  been observed and studied in the  linear function approximation case  \citep{de2003linear} in the context of approximate dynamic programming.
The corresponding dual formulation of the LP  \eqref{P00} can be written as follows: 
	\begin{equation}\label{P0}
	\begin{array}{ll}
	\max_{\theta}~~& r^\top\theta:=\sum_{s\in S,a\in A}r(s,a)\cdot \theta(s,a)\qquad\qquad 
	\mbox{s.t.~~~}   M\theta=(1-\gamma){\rho},~~~ \theta\ge 0, 
	\end{array}
	\end{equation} 
	where the matrix $M$ is defined as:
	$M:=\mathrm{Diag}(\bm{1}_{|A|}^\top,\cdots,\bm{1}_{|A|}^\top)-\gamma P.$
{Note that the optimal solution of the dual problem corresponds to the discounted occupancy measure of an optimal policy (see \cite{puterman1994markov}). Hence, we use the notation $\theta$ to denote the optimization variable of the dual problem.}

	We focus on solving the dual formulation \eqref{P0} in this paper.
	Then, the optimal $\theta^*$ can be used to generate a policy $\pi_{\theta^*}$, where $\pi_{\theta}$  is defined as
	\begin{equation}\label{pi_theta}
\pi_{\theta}(a\mid s)=\frac{\theta(s,a)}{\sum_{a'\in A}\theta(s,a')},
\end{equation}
{if $\sum_{a'\in A}\theta(s,a')>0$; and if $\sum_{a'\in A}\theta(s,a')=0$, $\pi_{\theta}(\cdot\mid s)$ can be any distribution in $\Delta(A)$, and we will defined it as a \emph{uniform} one with $\pi_{\theta}(a\mid s)=1/|A|$.}
This $\pi_{\theta^*}$ corresponds to an optimal policy $\pi^*$ of the MDP    \citep{puterman1994markov}.


To better study the relationship between the occupancy measure and the data distribution, we also consider the scaled version of the LP. This is also referred to as the \emph{marginal importance sampling}   formulation of the MDP in the offline RL literature \citep{nachum2019algaedice,lee2021optidice,zhan2022offline}. 
First, we define $w\in \mathbb{R}^m_+$ such that $w(s,a) \mu(s,a) = \theta(s,a)$, i.e., $w(s,a)$ denotes the ratio between the {occupancy} measure of the target policy and {the offline data distribution}. 

For each $(s,a,s')\in S\times A\times S$, let $K_{s',(s,a)}\in \mathbb{R}^{|S|\times m}$ be a matrix satisfying: i) for $s'\neq s$,  $K_{s',(s,a)}(s,(s,a))=1$, $K_{s',(s,a)}(s',(s,a))=-\gamma$; ii) for $s'=s$, $K_{s',(s,a)}(s,(s,a))=1-\gamma$; iii)   all other entries are zero. 
Define the distributions $\nu,\nu_{\Dc}\in\Delta(S\times A\times S)$ as follows: 
$\nu(s,a,s') :=P_{s,a}(s')\mu(s,a)$ and $\nu_{\Dc}(s,a,s') :=|n_{\Dc}(s,a,s')|/n$.
Finally, we also define the matrices
\begin{align}\label{equ:def_K_K_D}
	K =\EE_{(s,a,s')\sim \nu}K_{s',(s,a)}, \ \ \ \ \ \ \ \ \qquad K_{\Dc} =\EE_{(s,a, s')\sim \nu_{\Dc}}K_{s',(s,a)}.
\end{align}
Furthermore, we define $u\in \mathbb{R}^{m}$ such that   $u(s,a):= r(s,a) \mu(s,a).$ Then, we have the following lemma which relates these quantities to the ones in Problem \eqref{P0}.

\begin{Lemma}
\label{lemma:link}
We have $u^\top w = r^{\top} \theta$ and $K w = M\theta.$
\end{Lemma}
\proof{Proof.}
Note that the first inequality directly follows from the definitions of $u$ and $w$.

The second equality can be derived as follows. Let  $K(s',(s,a))$ and $M(s',(s,a))$ denote the $(s',(s,a))$-th element of the matrices $K$ and $M$, respectively. Note that  $K(s',(s,a))=M(s',(s,a))\cdot \mu(s,a)$ for all $(s,a,s')\in S\times A\times S$. Now:
$$
[K w]_{s} = \sum_{(\tilde{s}, \tilde{a})} K(s, (\tilde{s}, \tilde{a})) w (\tilde{s}, \tilde{a}) = \sum_{(\tilde{s}, \tilde{a})} M(s, (\tilde{s}, \tilde{a})) \mu(\tilde{s}, \tilde{a})w (\tilde{s}, \tilde{a})  = [M \theta]_{s}, 
$$
thereby completing the proof.
%
\hfill \moronly{$\square$}
\endproof

Using Lemma \ref{lemma:link}, we can rewrite Problem \eqref{P0} as follows:  
\begin{equation}\label{P:population1}
	\begin{array}{ll}
	\max_{w \ge 0} ~~u^\top w \qquad \text{s.t.~~~~}K w=(1-\gamma)\rho.
	\end{array}
\end{equation}
Let $w^*$ be the solution to Problem  \eqref{P:population1}, then we can obtain the optimal policy by computing $\pi_{w^*}$, where with a slight abuse of notation, $\pi_w$ is defined (similarly to $\pi_{\theta}$) as
\begin{align}\label{equ:def_pi_w}
	\pi_w(a\mid s):=\begin{cases}
      \frac{w(s,a)\pi_{\mu}(a \mid s)}{\sum_{a'\in A}w(s,a')\pi_{\mu}(a'\mid s)}, & \text{if~~~}{\sum_{a'\in A}w(s,a')\pi_{\mu}(a'\mid s) >0}\\
      \frac{1}{|A|} & \text{if~~~}{\sum_{a'\in A}w(s,a')\pi_{\mu}(a'\mid s) =0}
    \end{cases}.
\end{align}
We recall that $\pi_\mu$ is  the conditional distribution of $a$ given $s$ under $\mu$, which can also be viewed as the  {behavior policy} that generates the offline data. Note that for the $s\in S$ such that $\mu(s)=0$, $\pi_{\mu}(a\mid s)=0$ for all $a$  by definition, and thus $\pi_w(a\mid s)=1/|A|$, which is consistent with the one  from the normalization of $\theta$ from \eqref{pi_theta}, i.e., $\pi_{\theta}$, where  in this case, $\theta(s,a)=w(s,a)\mu(s,a)=0$ and thus  $\pi_\theta(a\mid s)=1/|A|$ for all $a\in A$. 

The equivalent primal-dual minimax reformulation of Problem \eqref{P:population1} is given by: 
\begin{equation}
\label{prob:minimax_pop}
\min_{w\in \mathbb{R}^m_+} \ \max_{v\in \mathbb{R}^{|S|}} \ ~~ -u^\top w+v^\top(Kw-(1-\gamma)\rho).
\end{equation}
Throughout the paper, we define
\begin{align}\label{equ:def_ell_pop}
	\ell(w,v):=-u^\top w+v^\top(Kw-(1-\gamma)\rho).
\end{align}

\subsubsection{Average-reward MDP}\label{sec:avg_MDP_formulation}

Similar to the discounted setting, one can reformulate solving the average-reward MDP as solving the following LP (mirroring Problem \eqref{P0}): 
\begin{equation}\label{equ:average_LP}
\begin{array}{ll}
&\max_{\theta\ge 0}~~~r^\top\theta\qquad\qquad\mbox{s.t.}~~~M\theta=0,\qquad 
\sum_{s\in S,a\in A}\theta(s,a)=1,	
\end{array}	
\end{equation}
and the scaled MIS version (mirroring \eqref{P:population1}):
\begin{equation}\label{equ:average_LP_MIS}
\begin{array}{ll}
&\max_{w\ge 0}~~~u^\top w 
\qquad\qquad \mbox{s.t.}~~~ Kw=0,\qquad \sum_{s\in S,a\in A}w(s,a)\mu(s,a)=1,
\end{array}	
\end{equation}
where $M:=\mathrm{Diag}(\bm{1}_{|A|}^\top,\cdots,\bm{1}_{|A|}^\top)- P$, and $K,K_{\Dc}$  are defined as in \eqref{equ:def_K_K_D}, with $K_{s',(s,a)}\in \mathbb{R}^{|S|\times m}$ being defined such that for each $(s,a,s')\in S\times A\times S$: i) if $s'\neq s$,  $K_{s',(s,a)}(s,(s,a))=1$, $K_{s',(s,a)}(s',(s,a))=-1$; 
ii)  all other entries are zero.  It is known that the optimal solution of \eqref{equ:average_LP} $\theta^*$ yields  the \emph{stationary distribution} of  an optimal policy $\pi_{\theta^*}$, which can also be obtained via \eqref{pi_theta}. Similarly, we can obtain an optimal policy $\pi_{w^*}$ from the optimal solution to Problem \eqref{equ:average_LP_MIS}, as in \eqref{equ:def_pi_w}. Hereafter, we will focus on the discounted MDP setting by default, and will explicitly note when switching to the average-reward case. 

\subsection{Empirical Formulation}\label{sec:empirical_form}

Since we do not have access to the exact distributions in the RL setting, we cannot solve  \eqref{prob:minimax_pop} directly.
Let $\hat{\rho}$ be an empirical estimate of $\rho$ (we can use $\hat{\rho}=\rho$ if $\rho$ is known). We thus define the following empirical counterpart of \eqref{equ:def_ell_pop}:
\begin{align}\label{equ:def_ell_emp}
\ell_{\Dc}(w,v):=-u_{\Dc}^\top w+v^\top(K_{\Dc}w-(1-\gamma)\hat{\rho}),	
\end{align}
where we recall the definition of $K_{\Dc}$ in \eqref{equ:def_K_K_D}, and define $u_{\Dc}\in \mathbb{R}^m$   with $u_{\Dc}(s,a)=r(s,a) \mu_{\Dc}(s,a)$, and  $\mu_{\Dc}(s,a)=n_{\Dc}(s,a)/n$.
We will then focus on the following empirical minimax optimization problem:
\begin{align}\label{P:emp0}
\min_{w\in \mathbb{R}^m_+} \ \max_{v\in \mathbb{R}^{|S|}} \ ~~ \ell_{\Dc}(w,v).
\end{align} 
Similarly, we can have the empirical version of the LP for average-reward  MDPs. We omit the details here and defer the discussions to \S\ref{sec:armdp}.   
\subsection{Function Approximation}
\label{app:func_approx}


{
To handle massively large state and action spaces, function approximation is usually used for the decision variables when solving MDPs, e.g., for  the value functions as well as the density ratios when one uses the LP framework as in \eqref{P:population1} or \eqref{prob:minimax_pop}. Note that for finite state-action spaces, the variables $v$ and $w$ are real vectors of dimensions $|S|$ and $|S||A|$, respectively. Following the convention in the literature \citep{chen2019information,chen2022offline,zhan2022offline,rashidinejad2022optimal}, we will refer to $v$ and $w$ as \emph{functions}, i.e., $v:S\to \mathbb{R}$ and $w:S\times A\to  \mathbb{R}$.
We will then use function classes $V$ and $W$ to approximate   $v$ and $w$, which usually have much smaller cardinality/covering number than the whole function spaces for all possible $v$ and $w$. The same convention also applies to other vectors of dimensions $|S|$ and/or $|S||A|$. We now introduce the following relationship of \emph{completeness} between two function classes, which will be used later in the analysis.}

{
\begin{Definition}[Completeness]\label{def:completeness}
	For two function classes $\mathcal{F}$ and $\mathcal{G}$, and a mapping $\phi:\mathcal{F}\to\mathcal{G}$, we say they satisfy \emph{$(\mathcal{F},\mathcal{G})$-completeness under $\phi$}, if for all $f\in\mathcal{F}$, $\phi(f)\in\mathcal{G}$.
\end{Definition}

Note that the common notion of {\it Bellman-completeness} \citep{munos2008finite,scherrer2014approximate,chen2019information} corresponds to the case where $\mathcal{F}=\mathcal{G}$ is the function class for approximating value functions, and $\phi$ is the Bellman operator \citep{bertsekas2017dynamic}.
}

In the following  sections, we propose a series of  offline RL algorithms under the LP framework, 
with function approximation to different variables. The overall goal is to design \emph{computationally tractable}  algorithms that achieve the \emph{optimal $O(1/\sqrt{n})$}  sample complexity, with \emph{partial data coverage}  (i.e., under the SPC or variant assumptions).  




\section{Optimal $O(1/\sqrt{n})$  Rate with Completeness-type  Assumption} 
 \label{sec:func_1}


We first solve offline RL with function approximation with an optimal (in terms of sample size $n$)  sample complexity of $O(1/\sqrt{n})$,  under the LP formulation, the  partial data coverage assumption, and 
 some completeness-type assumption. Throughout this section, we choose the distribution $\rho$ in the LP  reformulations in \S\ref{sec:LP_reform} as the initial state distribution $\mu_0$.

We start by introducing the following partial data coverage  assumption, as in \cite{rashidinejad2021bridging,rashidinejad2022optimal}. Note that it is made for the original MDP we would like to solve, and is weaker than the policy concentrability assumption in \cite{zhan2022offline}, which additionally includes the concentrability assumption for \emph{some regularized problem}. 

\begin{Assumption}[Single-Policy Concentrability]\label{ass:f1_conc}
There exists some constant $C^*>0$ such that for all $(s,a)\in S\times A$, $\frac{\theta_{\pi^*,\mu_0}(s,a)}{\mu(s,a)}\le C_{\pi^*,\mu_0}=: C^*$ for an  optimal policy  $\pi^*$. 
\end{Assumption}

Before proceeding further, we need some additional properties on the relationship between occupancy measure $\theta$ and the induced policy $\pi_\theta$, as shown next.

\subsection{Error Bound of the  $\pi_{\theta}$ Induced by $\theta$} 

Recall the definition of the 
occupancy measure induced by policy $\pi$ as $\theta_{\pi,\mu_0}\in\Delta(S\times A)$ (c.f. Eq. \eqref{equ:def_theta_pi}). 
Note that for simplicity,  we may omit the subscript $\mu_0$ in  $\theta_{\pi,\mu_0}$ throughout this section, as the initial distribution considered here is only $\mu_0$, and should be clear from the context.
Notice that a vector $\theta \in \mathbb{R}_+^m$ is not necessarily an occupancy measure of any  policy $\pi$. The first lemma below  shows that $\theta$ is an occupancy measure if it satisfies the constraints in Problem \eqref{P0} with $\rho=\mu_0$. 

\begin{Lemma}\label{feasible}
If some  $\theta\in \mathbb{R}^m$ satisfies $\theta\ge 0$ and $M\theta=(1-\gamma)\mu_0$,
we have $\theta=\theta_{\pi_{\theta}}$, where we recall the definition of $\pi_{\theta}$ in \eqref{pi_theta}. 
Moreover, in this case, we have $J_{\mu_0}(\pi_{\theta})=r^\top \theta.$
\end{Lemma}

This lemma is a special case of the next lemma and also  a result in Chapter 6.9 of \citep{puterman1994markov}.
The next question is how close $\theta$ is to $\theta_{\pi_{\theta}}$ if $\theta$ is not in the set $\{\theta\mid  M\theta=(1-\gamma)\mu_0,\theta\ge 0\}$, i.e., it does not satisfy the constraints. The following lemma provides an \emph{error bound}  that relates the \emph{occupancy validity  constraint violation} and the  \emph{absolute  difference} between $r^\top\theta$ and $r^\top\theta_{\pi_{\theta}}$.

\begin{Lemma}
\label{lemma:infeasible}
For any $\theta\ge 0$,
we have $|r^\top(\theta-\theta_{\pi_{\theta}})|\le \frac{\|M\theta-(1-\gamma)\mu_0\|_1}{1-\gamma}$.
\end{Lemma}
	 
The term $r^\top\theta_{\pi_{\theta}}$ in Lemma \ref{lemma:infeasible} exactly corresponds to $J(\pi_{\theta})$.
Next, we introduce the following definition.



{
\begin{Definition}[Sign Function]\label{def:x_sign}
	For any $w\in\mathbb{R}_+^m$, we define the  mapping $\phi:\mathbb{R}^{m}\rightarrow \mathbb{R}^{|S|}$ such that  $\phi(w)\in \arg\max_{x:\|x\|_{\infty}\le 1}x^\top(Kw-(1-\gamma)\mu_0)$ as the \emph{sign function} of the occupancy validity  constraint $Kw-(1-\gamma)\mu_0=0$ in \eqref{P:population1}. In particular, note that $\phi(w)^\top (Kw-(1-\gamma)\mu_0) = \|Kw-(1-\gamma)\mu_0\|_1$. 
\end{Definition}}

{By the definition of dual norm, we refer to $\phi(w)$ as the \emph{sign function}, where  we follow the convention that the sign of $0$ can be any  $x$ with $\|x\|_\infty\leq 1$. We are now ready to state our assumption on the function classes.


\begin{Assumption}
\label{ass:f1_fake}
Let $x_w:=\phi(w)$ with $\phi$ given in Definition \ref{def:x_sign}. 
Let $W$ and $B$ be the function classes for $w$ and $x_w$, respectively.  
Then, we have realizability of $W$, and $(W,B)$-completeness under $\phi$, 
  i.e., $w^*\in W$ for the optimal $w^*$ such that $\pi_{w^*}=\pi^*$ satisfies  Assumption \ref{ass:f1_conc} \footnote{Note that this implies $1\leq C^* \leq B_W$.},   and  $x_w\in B$ for all $w\in W$. 
  Furthermore, $W$ and $B$ are bounded, i.e., $w\geq 0$ and $\| w \|_\infty \leq B_W$ 
  for all $ w \in W$ and 
   $\| x \|_\infty \leq 1$ for all $x \in B$.
\end{Assumption}}

\begin{Remark}
Several remarks are in order. {First, we introduce $x_w$ to calculate   the $\ell_1$-norm of $Kw-(1-\gamma)\mu_0$, since the $\ell_1$-norm  will be related to the suboptimality gap of the policy obtained from $w$ (see Lemma \ref{lemma:infeasible} above).}  
Second, Assumption \ref{ass:f1_fake} contains not only realizability of $W$ for $w^*$, but also the {\it completeness-type}  assumption of $B$ for $x_w\in B$ for any $w\in W$. The completeness-type assumptions are standard in the offline RL  literature \citep{munos2008finite,chen2019information,xie2021bellman}, and can be challenging or even impossible to remove in certain cases due to some  hardness results \citep{foster2021offline}. To the best of our knowledge, the only existing results that \emph{merely  assume realizability} of the optimal solutions are  \cite{uehara2021pessimistic,zhan2022offline,chen2022offline}, which are either statistically sub-optimal or computationally intractable. We defer our solution to the \emph{realizability-only} case to \S\ref{sec:second_FA}.

Third, and most interestingly,  
some \emph{completeness assumption} is also made in the concurrent and independent work \cite{rashidinejad2022optimal} that achieves the optimal $O(1/\sqrt{n})$ rate as well (see their Theorem 4), in the same setup as in this section (with general function approximation and the LP framework, and with partial data coverage assumption as to be introduced next).  The completeness-type assumption therein mirrors  our Assumption \ref{ass:f1_fake}.  Note that we only need the completeness of \emph{one}  function class $B$ for $x_w$, while \cite{rashidinejad2022optimal}  requires the completeness of $U$ for $u_w^*$ (with the notation therein), the realizability of $v^*$, together with either the realizability of the model $P$ (which is deemed as even stronger than Bellman-completeness \citep{chen2019information,zhan2022offline,uehara2021pessimistic}), or the completeness of two  function classes $\mathcal{U}$ and $\mathcal{Z}$ therein.
\end{Remark}


\subsection{A Reformulated LP \& Main Result}

Now we can state our approach. From Lemma \ref{lemma:infeasible}, we know that if we control $\|M\theta-(1-\gamma)\mu_0\|_1=\|Kw-(1-\gamma)\mu_0\|_1$, we can make the inner product $r^\top \theta$ be close to the actual reward under the policy $\pi_\theta$, i.e., $r^\top \theta_{\pi_\theta}$. This motivates us to add a constraint to control $\|Kw-(1-\gamma)\mu_0\|_1$.

Recall that $u_{\Dc}\in \mathbb{R}^m$ is defined as $u_{\Dc}(s,a)=r(s,a) \mu_{\Dc}(s,a)$. 
Our approach is to solve the following {LP-based optimization problem constructed from the dataset $\Dc$}:
\begin{equation}\label{alg2}
	\begin{array}{ll}
&\max_{w\in W}~~~~u_{\Dc}^\top w ~~~~~\qquad \text{s.t.}~~~~~~x^\top(K_{\Dc}w-(1-\gamma)\mu_0)\le E_{n,\delta}, ~~ \forall x \in B,
\end{array}
\end{equation}
where $E_{n,\delta}:=\frac{2B_W\sqrt{2\log (|B||W|/\delta)}}{\sqrt{n}}$. Program \eqref{alg2} can be viewed as a \emph{relaxation} of the empirical version of Problem \eqref{P:population1}, by relaxing the constraint $K_{\Dc}w-(1-\gamma)\mu_0=0$ to an appropriate inequality. 


Suppose that we have a solution to Problem \eqref{alg2}, denoted by  $w_{\Dc}$, then we can obtain the policy $\pi_{\Dc}$ by setting  $\pi_{\Dc}=\pi_{\tilde{\theta}_{\Dc}}$, where for each $(s,a)\in S\times A$ 
\begin{equation}\label{equ:def_tilde_theta}
	\begin{array}{ll}
	\tilde{\theta}_{\Dc}(s,a)=w_{\Dc}(s,a)\pi_{\mu}(a\mid s).
\end{array}
\end{equation}	
The performance of the policy $\pi_{\Dc}$ is given in the following theorem, with proof deferred to \S\ref{append:func_1}:

\begin{Theorem}
\label{main1}
Suppose Assumptions \ref{ass:f1_conc} and \ref{ass:f1_fake} hold. Then,  with probability at least $1-6\delta$, the policy $\pi_{\Dc}$ learned from solving \eqref{alg2} satisfies 
\begin{align*}
J_{\mu_0}(\pi^*)-J_{\mu_0}(\pi_{\Dc})\le \frac{4\sqrt{2}B_W\sqrt{\log (|B||W|/\delta)}}{(1-\gamma)\sqrt{n}}.
\end{align*}
\end{Theorem}


Note that Theorem \ref{main1} gives an optimal sample complexity of $O(1/\sqrt{n})$ in terms of the sample size $n$, with general function approximation. Compared to the recent work \cite{zhan2022offline} that also uses the LP framework for offline RL, we exchange the realizability assumption on $v^*$ therein for some completeness-type assumption, while improving the sample complexity from $O(1/n^{1/6})$ to $O(1/\sqrt{n})$. Compared to other offline RL algorithms with general function approximation that have the optimal $O(1/\sqrt{n})$ rate, e.g., \cite{xie2021bellman,uehara2021pessimistic,chen2022offline}, which are computationally intractable, our algorithm is tractable if the function classes are  convex,   inheriting the computational advantage of the LP framework for offline RL \citep{zhan2022offline,rashidinejad2022optimal}. 

Finally, compared with the independent work \cite{rashidinejad2022optimal}, both of the works require the realizability of $W$ for $w^*$ and some completeness-type assumptions (we need one such assumption while they need two, which may not be comparable as the function classes used are different). 
 Moreover, we do not need the realizability of $v^*$ and have better $(1-\gamma)^{-1}$ dependence (i.e., $(1-\gamma)^{-1}$ v.s. $(1-\gamma)^{-3}$), with a   simpler algorithm and analysis. We note that the key to obtain $O(1/\sqrt{n})$ rate in \cite{rashidinejad2022optimal} is also to enforce the constraint of $w_\Dc$, where they use the technique of augmented Lagrangian, while we introduce a \emph{constrained}  LP  directly. 

\begin{Remark}
Note that when $W,B$ are continuous sets, Theorem \ref{main1} can still be true if we replace the cardinality by the covering number or the number of extreme points (see \S\ref{sec:sparse_case} for a concrete  example). Then, if $W$ and $B$ are convex, our algorithm is computationally tractable since   it is  solving a {\it convex} program. 
\end{Remark} 


\subsection{Tabular Case}\label{sec:discounted_mdp_tabular}

In this subsection, we show that our  framework  above can be directly instantiated to the  \emph{tabular} case without function approximation, while maintaining the optimal $O(1/\sqrt{n})$ sample complexity.

Here we need realizable function classes $W$ and $B$.
To make the algorithm computationally tractable, we use continuous and convex function classes   $W$ and $B$,  instead of discrete, finite ones. In particular, let $W= \{w\in \mathbb{R}_+^m\mid \sum_{a\in A}w(s,a)\le B_W, \forall s\in S\}$ 
and $B=[-1,1]^{|S|}$, which are convex and compact, and satisfy the boundedness assumptions in Assumption \ref{ass:f1_fake}. 
Then we solve the following convex program: 
\begin{align}\label{equ:alg_tabular_discounted}
&\max_{w\in W}~~~~u_{\Dc}^\top w, \nonumber \\
&\text{s.t.}~~~~\max_{x\in B}~x^\top(K_{\Dc}w-(1-\gamma)\mu_0)=\|K_{\Dc}w-(1-\gamma)\mu_0\|_1   \le \frac{2B_W\sqrt{|S|\log((2|A|+2)/\delta)}}{\sqrt{n}}.  
\end{align}

We thus have the following theorem, whose proof can be found in \S\ref{append:func_1}:
\begin{Theorem}\label{tabular_main}
Suppose Assumption \ref{ass:f1_conc} holds, and the MDP is non-degenerate  in the sense that $\min\{|A|,|S|\}>1$.  
Then, with probability at least $1-6\delta$, the policy $\pi_{\Dc}$ learned from solving \eqref{equ:alg_tabular_discounted} satisfies  
\begin{align*}
J_{\mu_0}(\pi^*)-J_{\mu_0}(\pi_{\Dc})\le \frac{4B_W\sqrt{|S| \cdot\log((2|A|+2)/\delta)}}{(1-\gamma)\sqrt{n}}.
\end{align*} 
\end{Theorem}

Compared to the independent  and most related  result \cite{rashidinejad2022optimal}, 
our result yields a better sample complexity when reduced to  the tabular case. In particular,
\cite{rashidinejad2022optimal} leads to a $\tilde O(1/((1-\gamma)^4\sqrt{n}))$ rate\footnote{Note that we did not specify the dependence on $|S|$ explicitly here, since it depends on the function classes being used in \cite{rashidinejad2022optimal} (which are different from ours and not comparable), which we believe can be of order $\sqrt{|S|}$ as ours.  We thus only focus on the dependence of $1/(1-\gamma)$ and $n$.  
Also, the additional $(1-\gamma)^{-1}$ comes from the fact that $B_v$ therein is of order $(1-\gamma)^{-1}$. 
 }, while we have $\tilde O(1/((1-\gamma)\sqrt{n}))$. In fact, our reduction is comparable to and even better than (in terms of $(1-\gamma)$) the state-of-the-art result \emph{specifically for the tabular} case \citep{rashidinejad2021bridging}, which is\footnote{Note that the definition of $J_{\mu_0}$ in \cite{rashidinejad2021bridging} is $(1-\gamma)$-factor off from our definition.} $\tilde O(\sqrt{|S|}/((1-\gamma)^{3/2}\sqrt{n}))$, while ours is $\tilde O(\sqrt{|S|}/((1-\gamma)\sqrt{n}))$. 
  Finally, note that the lower-bound for the tabular case is $\Omega(\sqrt{|S|}/(\sqrt{(1-\gamma)n}))$ in  \citep{rashidinejad2021bridging}, and we believe that using the Bernstein's (instead of Hoeffding's) concentration inequality together with some   variance reduction technique \citep{sidford2018near} may  further improve our dependence on $1/(1-\gamma)$, and attain the lower bound. 
  We leave these directions of improvement   as future work, since our focus is on the general function approximation case, and on the optimality of sample complexity in terms of $n$. We use the tabular case to simply demonstrate the power and generality  of our constrained-LP-based framework. Next, we show that our framework can also readily handle the average-reward  MDP setting, a less studied one in offline RL. 
  
\subsection{Sparse Visitation Case}\label{sec:sparse_case}
   
In this subsection, we will consider another case that our LP-framework may be beneficial -- when the number of state-actions visited by an optimal policy is \emph{small}, i.e, the \emph{sparse visitation} case. We will show that our framework can readily  capture and leverage this structure, and yield a better sample complexity that was not covered by some generic approaches \cite{rashidinejad2021bridging,rashidinejad2022optimal,zhan2022offline}. In particular, we first make the following assumption.

\begin{Assumption}\label{sparse} 
Suppose we have realizability of $W$, i.e., $w^*\in W$ for an  optimal $w^*$ such that the corresponding optimal policy $\pi_{w^*}=\pi^*$ satisfies  Assumption \ref{ass:f1_conc}. Also, let    $S_*:=\left\{s~\vert ~\sum_{a\in A}\theta_{\pi_{w^*},\mu_0}(s,a)\ne 0, s\in S\right\}$, then suppose  $|S_*|\leq q$, i.e.,   $\pi_{w^*}$ only visits \emph{at most}  $q$ states starting from $\mu_0$ (we do not need to know \emph{which} $q$ states). 
Furthermore,  $W$ is bounded, i.e., $w\geq 0$ and $\| w \|_\infty \leq B_W$ 
  for all $ w \in W$.  
\end{Assumption}


 Then, defining $S(w):=\{s~\vert~\sum_{a\in A}w(s,a)\mu(s,a)\ne 0, s\in S\}$, it does not lose optimality to update   $W=W\setminus\{w~\vert~|S(w)|> q\}$, i.e., to only consider those $w\in W$ such that $\sum_{a\in A}w(s,a)\mu(s,a)$ at most visits $q$ states.  For such a $W$, we have the following proposition.

\begin{Proposition}\label{sparsity}
Suppose Assumptions  \ref{ass:f1_conc} and \ref{sparse} hold. Update $W=W\setminus\{w~\vert~|S(w)|> q\}$. 
Then, letting $B=\{x\in \{-1,1\}^{|S|}\mid \|x+\mathbf{1}^{|S|}\|_0\le q\}$, we have that  $(W,B)$ is complete under $\phi$, i.e., for any $w\in W$, $\phi(w)\in B$. 
Moreover, $|B|\le |W|\cdot 2^q$.
\end{Proposition}


Combining Proposition \ref{sparsity} and Theorem \ref{main1}, 
we have the following result. 

\begin{Corollary}
Suppose Assumptions  \ref{ass:f1_conc} and \ref{sparse} hold. Then,  with probability at least $1-6\delta$, the policy $\pi_{\Dc}$ learned from solving \eqref{alg2} with $W$ as defined in Proposition \ref{sparsity} satisfies 
$$J_{\mu_0}(\pi^*)-J_{\mu_0}(\pi_{w_{\Dc}})\le \tilde{O}\left(\frac{B_W\sqrt{
q\log(2(|W|/\delta)^{2/q})}}{(1-\gamma)\sqrt{n}}\right).$$
\end{Corollary}

Compared to the result in Theorem \ref{main1'} for the tabular case, the sample complexity does not have explicit dependence on the state space cardinality $|S|$ anymore, but only on the number of states \emph{visited} by the optimal policy covered by the offline data $q$, which may satisfy $q\ll |S|$. 


\section{Average-reward MDP Setting}\label{sec:armdp}

We now  generalize  our LP-framework to  average-reward MDP settings. We first make the following natural and widely used regularity assumption on the mixing  time of the MDPs  \citep{wang2017primal,jin2020efficiently,jin2021towards,li2022stochastic,abbasi2019politex,wei2020model}, under which the average-reward objective \eqref{equ:obj_avg_mdp} is well-defined.

\begin{Assumption}\label{mix}
There exists some $T_0>0$ such that   $\|P_{\pi}^{T_0}\beta-\beta_{\pi}\|_1\le 1/2$ for any policy $\pi$ and any state distribution $\beta\in\Delta(S)$, where $P_{\pi}:S\to\Delta(S)$ with $P_{\pi,s}(s')=\sum_{a\in A}\pi(a\mid s)P_{s,a}(s')$,  and $\beta_\pi\in\Delta(S)$ is the stationary state-distribution under $\pi$.    
\end{Assumption}

To address offline RL, following \S\ref{sec:func_1}, we also make the following partial data coverage assumption: the offline data distribution should at least cover an optimal policy $\pi^*$. 

\begin{Assumption}[Single-Policy Concentrability]\label{ass:SPC_avg}
There exists some constant $C^*>0$ such that for all $(s,a)\in S\times A$, $\frac{\theta_{\pi^*}(s,a)}{\mu(s,a)}\le C_{\pi^*}=:C^*$ for an  optimal policy  $\pi^*$, where  $\theta_{\pi^*}$ is the stationary state-action-distribution under the policy $\pi^*$. 
\end{Assumption}

\begin{Remark}[A Naive Approach: Approximation Using Discounted MDPs]
	As shown in many results referenced in \S\ref{app:sec_1_1} (including our result in  \S\ref{sec:func_1}), 
 offline RL  in discounted MDPs can be addressed  with $\frac{1}{\sqrt{\mathrm{poly}((1-\gamma)n)}}$ sample complexity. 
Moreover, under Assumption \ref{mix}, it is known that there is a $(1-\gamma)$ gap between the optimal solution  in discounted MDP and that in average-reward MDP  \cite{jin2021towards,wang2022near}. Thus, one can take a large enough $\gamma$ such that  $1-\gamma\le {O}(1/\sqrt{n})$, so that the gap is dominated by the statistical error of order ${O}(1/\sqrt{n})$.  
However,  this naive approach will result in at best ${O}(1/n^{1/4})$ sample complexity for the average-reward case (if one plugs in even the best-so-far  sample complexity of ${O}\left(\frac{1}{(1-\gamma)\sqrt{n}}\right)$, as given by our Theorem \ref{tabular_main}). Hence, additional techniques are needed to  obtain the optimal ${O}(1/\sqrt{n})$ sample complexity. 
\end{Remark}

\subsection{A Reformulated LP \& Main Result}

Based on the MIS formulation in \S\ref{sec:avg_MDP_formulation}, it is tempting to consider the empirical version of \eqref{equ:average_LP_MIS}, by estimating $K$ and $u$ from data. We use $K_{\Dc}$ and $u_{\Dc}$ to denote their estimates, where we recall that $K_{\Dc}$ is defined as in \eqref{equ:def_K_K_D}, and $u_{\Dc}\in \mathbb{R}^m$ is defined  as $u_{\Dc}(s,a)=r(s,a) \mu_{\Dc}(s,a)$ with $\mu_{\Dc}(s,a)=n_{\Dc}(s,a)/n$ for all $(s,a)$. Then, the empirical version of Program \eqref{equ:average_LP_MIS} can be written as 
\begin{equation}\label{empirical00}
	\begin{array}{ll}
\max_{w\ge 0}~~u_{\Dc}^\top w \qquad ~~~ \text{s.t.}~~~~~~K_{\Dc}w=0,~~~ ~~~~~\sum_{s\in S,a\in A}w(s,a)\mu_{\Dc}(s,a)=1.
\end{array}
\end{equation}

To ensure that the empirical and the population versions of LP are closed to each other, we need to make use of concentration inequalities for $(K-K_{\Dc})w, (u-u_{\Dc})^\top w$, which generally require certain boundedness of $w$. 
Let $B_W>0$ be the bound such that $\|w\|_{\infty}\le B_W$, then one may consider adding a constraint $\|w\|_\infty\leq B_W$ in Program \eqref{empirical00}. 
However, such a new program built upon empirical data  may \emph{not} have a  feasible solution due to this constraint, if $B_W$ is not chosen properly. To address this issue, we further \emph{relax} the equality constraints (i.e., the \emph{stationary distribution validity  constraints}) in \eqref{empirical00}, following the idea in \S\ref{sec:func_1}. We thus consider solving the following program:
\begin{equation}\label{empirical}
	\begin{array}{ll}
\max_{w\in W}~~u_{\Dc}^\top w   \qquad ~ \text{s.t.}~~~~~~x^\top K_{\Dc}w\le E_{n,\delta}, ~~\forall x\in B,~~~~~ \big|\sum_{s,a}w(s,a)\mu_{\Dc}(s,a)-1\big|\le E_{n,\delta},
\end{array}
\end{equation}
where $E_{n,\delta}:=\frac{2B_W\sqrt{2\log (|B||W|/\delta)}}{\sqrt{n}}$. 
 By a slight abuse of notation, let $w_{\Dc}$ be the solution to Problem \eqref{empirical}. Then, one can output a policy 
$\pi_{\Dc}$ using $w_{\Dc}$  by setting  $\pi_{\Dc}=\pi_{\tilde{\theta}_{\Dc}}$, with $\tilde{\theta}_{\Dc}$ given as in \eqref{equ:def_tilde_theta} and $\pi_\theta$ given as in \eqref{pi_theta}.  
Note that Program \eqref{empirical} is a convex and thus tractable optimization problem.

\begin{Remark}
Unlike the discounted case, the last inequality in \eqref{empirical} is needed as $Kw=0$ only does not necessarily imply the corresponding $\theta$ is a stationary distribution. One needs to make sure $\theta$ is nearly a valid \emph{probability distribution}.
\end{Remark}

Following \S\ref{sec:func_1}, we make the following assumption on the function classes $W$ and $B$.

\begin{Assumption}
\label{ass:f1_fake_avg}
Let $x_w:=\phi(w)\in\arg\max_{x:\|x\|_{\infty}\le 1}x^\top Kw$ be the \emph{sign function} of the validity  constraint $Kw=0$ in \eqref{equ:average_LP_MIS}.   
Let $W$ and $B$ be the function classes for $w$ and $x_w$, respectively.  
Then, we have realizability of $W$, and $(W,B)$-completeness under $\phi$, 
  i.e., $w^*\in W$ for the optimal $w^*$ such that $\pi_{w^*}=\pi^*$ for the $\pi^*$ in Assumption \ref{ass:SPC_avg},   and  $x_w\in B$ for all $w\in W$. 
  Furthermore, $W$ and $B$ are bounded, i.e., $w\geq 0$ and $\| w \|_\infty \leq B_W$ for all $ w \in W$ and 
   $\| x \|_\infty \leq 1$ for all $x \in B$.
\end{Assumption} 

Now we 
present the guarantee for the  policy $\pi_{\Dc}$ learned from  \eqref{empirical}, whose proof can be found in \S\ref{append:avg_mdp}.  

\begin{Theorem}
\label{main1_avg}
Suppose Assumptions \ref{mix}, \ref{ass:SPC_avg}, and \ref{ass:f1_fake_avg} hold. Then, for $n\geq 32 B_W^2 \log(|B||W|/\delta)$,  with probability at least $1-6\delta$, the policy $\pi_{\Dc}$ learned from solving \eqref{empirical} satisfies 
\begin{align*}
J(\pi^*)-J(\pi_{\Dc})\le O\left(\frac{B_WT_0\sqrt{\log (|B||W|/\delta)}}{\sqrt{n}}\right).
\end{align*}
\end{Theorem}

Theorem \ref{tabular_avg_main}  shows that by properly relaxing the occupancy  validity constraint, one can obtain $\tilde O(1/\sqrt{n})$ sample complexity with the SPC assumption and computational tractability, for offline RL with general function approximation. To the best of our knowledge, this is among the first results on offline RL  with partial data coverage and function approximation (together with \cite{gabbianelli2024offline}). Compared to the independent work \cite{gabbianelli2024offline} that can also address the average-reward setting, and focused on the linear function approximation case,   our rate is better than the  $\tilde O(1/{n}^{1/4})$ therein. It is worth noting  that our Assumption \ref{mix}, though standard, is stronger than the assumption used  in \cite{gabbianelli2024offline}, which assumes bounded span of the bias function. Hence, the results are  not completely comparable. We will leave the handling of such a weaker assumption within our dual-LP framework  and for \emph{general}  function approximation (while maintaining the optimal $\tilde O(1/\sqrt{n})$ rate) as an important future work.

\subsection{Tabular Case}

Similar to the discounted setting, our framework can also be readily instantiated to the tabular case for average-reward MDPs. Following \S\ref{sec:discounted_mdp_tabular}, we choose $W:= \{w\in \mathbb{R}_+^m\mid \sum_{a\in A} w(s,a)\le B_W, \forall s\in S\}$ 
and $B:=[-1,1]^{|S|}$, and solve the following convex program:
\begin{align}\label{equ:alg_tabular_avg}
\max_{w\in W}~~u_{\Dc}^\top w   \qquad\quad  ~ \text{s.t.}~~~\max_{x\in B}~x^\top K_{\Dc}w=\|K_{\Dc}w\|_1\le E_{n,\delta},~~~~~ \bigg|\sum_{s,a}w(s,a)\mu_{\Dc}(s,a)-1\bigg|\le E_{n,\delta},  
\end{align}
where $E_{n,\delta}:=2B_W\sqrt{|S|\log((2|A|+2)/\delta)}/{\sqrt{n}}$. We have the following guarantee for $\pi_{\Dc}$ obtained from \eqref{equ:alg_tabular_avg}.  
\begin{Theorem}\label{tabular_avg_main}
Suppose Assumptions  \ref{mix} and \ref{ass:SPC_avg} hold. 
Then, for $n\geq 16 B^2_W |S|\log((2|A|+2)/\delta)$, with probability at least $1-6\delta$, the policy $\pi_{\Dc}$ learned from solving \eqref{equ:alg_tabular_avg} satisfies
\begin{align*}
J(\pi^*)-J(\pi_{\Dc})\le \tilde O\left(\frac{B_WT_0\sqrt{|S|\log((2|A|+2)/\delta)}}{\sqrt{n}}\right).
\end{align*} 
\end{Theorem} 

To the best of our knowledge, Theorem \ref{tabular_avg_main} is the first offline RL result for tabular average-reward MDP with $O(\sqrt{S/n})$ rate and partial data coverage (see also a preliminary version in the workshop version of the present paper \cite{ozdaglarpower}). As in the function approximation case above, it is an important future work to maintain the $O(\sqrt{S/n})$ rate under the weaker bounded-span assumption (instead of under Assumption \ref{mix}).


\begin{Remark}[Exploiting Sparsity Visitation] 
Similar to \S\ref{sec:sparse_case}, one can also exploit the \emph{sparsity} of the states visited by the optimal policy covered by the offline data to obtain an improve bound, which does not have an explicit dependence on $|S|$, but the number of visited states $q\ll |S|$. The key is to leverage the sparsity in constructing the sign function class $B$. 
\end{Remark}

We have shown  that it is possible to obtain the optimal $O(1/\sqrt{n})$ rate with computational tractability for offline RL with function approximation, 
thanks to our properly-constrained-LP framework. However, some \emph{completeness-type} assumption is still needed. 
This naturally raises the question:
\begin{center}
	{\it Can we have computationally tractable offline RL algorithms with  $O(1/\sqrt{n})$  rate, 
	 but with \\ only  realizability and partial data coverage assumptions?}
\end{center}
which was the open question left in the literature \citep{zhan2022offline,chen2022offline}. 
Next, we investigate this question, again under a  LP framework with proper constraints.  

\section{Gap-Dependent  $O(1/\sqrt{n})$ Rate with Realizability-only Assumption}
\label{sec:second_FA}

In this section, we return to the default discounted MDP setting, and 
solve the LP-induced   minimax optimization \eqref{prob:minimax_pop}, with  function approximation to $v$ and $w$. 
Notice that this setting is also considered in \cite[Section  4.5]{zhan2022offline}. It is also related to the setting in \cite{chen2022offline}, where function approximation was used for the state-action function $Q$ and $w$.




We select $w,v$ from  finite sets\footnote{As in several related works  \cite{zhan2022offline,rashidinejad2022optimal}, in the case they are infinite classes, we can replace the results in this section with a standard covering argument.} $W,V$. 
{Throughout this section, we sometimes write $\theta_{\pi,\rho}$ as $\theta_{\pi}$ for notational convenience. Also, we specify the $\rho$ in the LP formulation \eqref{P:population1} and  \eqref{prob:minimax_pop} as  $\rho(s)=\mu(s)$ for all $s\in S$  throughout this section, unless otherwise noted. 
In the end of this section, we will relate the return $J_{\rho}(\pi)$ back to $J_{\mu_0}(\pi)$.}
{Note that we only have}  access to the empirical version $\mu_{\Dc}$ of $\mu$, where  $\mu_{\Dc}(s,a)=n_{\Dc}(s,a)/n$. 
The next proposition specifies a lower bound of $\theta_{\pi^*,\rho}$: 
\begin{Proposition}\label{lower}
For any  optimal policy $\pi^*$ {and any initial state distribution $\rho\in\Delta(S)$}, we have $\sum_{a\in A}\theta_{\pi^*,\rho}(s,a)\ge (1-\gamma)\cdot \rho(s)$ for all $s\in S$.  
\end{Proposition}
This is a direct corollary of the fact that for any policy $\pi$, by definition we have $\sum_{a\in A}\theta_{\pi,\rho}(s,a)\ge (1-\gamma)\sum_{a\in A}\rho(s,a)=(1-\gamma)\rho(s)$ for all $s\in S$. 
{In particular, we would like to note that Proposition \ref{lower} is true for the initial state distribution $\rho(s)=\mu(s)$.}

The design of algorithms in this section is based on the following intuitive idea: 
According to Eq.  \eqref{equ:def_pi_w}, for the $w$ such that $\sum_aw(s,a)\pi_{\mu}(a\mid s)=0$, the policy $\pi_w$ has to be assigned randomly (as a uniform distribution for example),  and cannot be decided from the offline data. 
To avoid this case, one direct approach is to add a {\it lower bound  constraint} to the vanilla minimax problem~\eqref{prob:minimax_pop}. Specifically, we consider the following population minimax problem:
\begin{equation}
	\begin{array}{ll}
&\min\limits_{w\in \mathbb{R}_+^m}\max\limits_{v\in \mathbb{R}^{|S|}} \ ~~~~ -u^\top w+v^\top(Kw-(1-\gamma){\mu}) \label{alg3:pop:tabular}\\
&\text{s.t.}~~~~~~\qquad\quad\sum_{a\in A}w(s,a)\pi_{\mu}(a\mid s)\ge (1-\gamma),~~~\forall s\in S.
	\end{array}
\end{equation}
Note that compared to the vanilla minimax problem~\eqref{prob:minimax_pop}, the only difference is that we enforce the lower bound constraints on $\sum_{a\in A} w(s,a)\pi_{\mu}(a\mid s)$. This lower bound constraint, along with the upper bound shown in Lemma \ref{inactive}, will help control the probability of choosing an \emph{inactive} state-action pair by the policy generated by the solution of \eqref{alg3:pop:tabular}. 
Furthermore, by Proposition \ref{lower}, we know that for those $s\in S$ such that $\mu(s)>0$, 
the policy $\pi_{\tilde w^*}$ induced by the solution to \eqref{alg3:pop:tabular} $\tilde w^*$ (via \eqref{equ:def_pi_w}) is also an optimal policy, since the optimal $w^*$ also satisfies the lower bound constraint. For those $s\in S$ but $\mu(s)=0$, we will show that such a constraint will not affect the optimality of the induced policy $\pi_{\tilde w^*}$, either, under certain assumptions to be introduced later. Overall, the lower bound constraint will not eliminate the possibility of finding an optimal policy, when there is infinite data. 


 Hence, we turn to solving  \eqref{alg3:pop:tabular} using  function approximation, i.e., our algorithm is  to solve: 
\begin{equation}\label{alg3:pop:fa}
\min_{w\in W}\max_{v\in V} \ ~~~~~ -u^\top w+v^\top(Kw-(1-\gamma){\mu}) ,
\end{equation} 
where $W$ is defined as
\begin{align}\label{equ:def_W}
W:=\left\{w  \biggm|   w\in \mathbb{R}_+^m, \sum_{a\in A}w(s,a)\pi_{\mu}(a\mid s)\ge (1-\gamma), \forall s\in S\right\}. 
\end{align}   

To learn an approximate optimal policy from the offline data, we solve the following empirical version of the minimax problem in \eqref{alg3:pop:fa}:
\begin{align}\label{alg3}
&\min_{w\in W}\max_{v\in V} ~~~~~ -u_{\Dc}^\top w+v^\top(K_{\Dc}w-(1-\gamma)\mu_{\Dc}),
\end{align}
whose solution is denoted by $w_{\Dc}$. Similar as before, a  policy $\pi_{\Dc}$ is then obtained by letting $\pi_{\Dc}=\pi_{\tilde{\theta}_{\Dc}}$, with $\tilde{\theta}_{\Dc}$ given as in \eqref{equ:def_tilde_theta} and $\pi_\theta$ given as in \eqref{pi_theta}. 



\subsection{Assumptions}

Before moving to the main theoretical result in this section, we first state our assumptions and some additional notation. 
We first make the realizability assumptions for the function classes $W$ and $V$.

\begin{Assumption}[Realizability and Boundedness of $W$] \label{realizability}
There exists some solution $w^*\in W$ solving  \eqref{alg3:pop:tabular}, 
and hence also solving \eqref{prob:minimax_pop} with $\rho=\mu$. 
Moreover, we suppose $w\geq 0$ and $\|w\|_\infty\le B_W$ for all $w\in W$.
\end{Assumption}

\begin{Assumption}[Realizability and Boundedness of $V$]\label{v_realizable}
Suppose that $v^*\in V\subseteq [-1/(1-\gamma),1/(1-\gamma)]^{|S|}$.
\end{Assumption} 

Notice that similar assumptions are used in \cite{zhan2022offline,chen2022offline}. 
Next, we make the assumption regarding partial data coverage, which suggests  that the offline data should cover some single optimal policy. 
For ease of presentation, we first introduce the following definitions. 

\begin{Definition}
We denote by $S_0$, the set of states visited by the  offline data distribution $\mu$, i.e., $S_0:=\{s\in S\mid \mu(s)>0\}$, where we recall that $\mu(s)=\sum_{a\in A}\mu(s,a)$ for any $s\in S$ \footnote{Note that we do not need to know $S_0$ for our algorithm to be introduced  later. We only need the definition of $S_0$ for analysis.}. Also, for any policy $\pi$ and any $s\in S$, we define
\begin{align*}
	S_{\pi}(s) :=\{a\in A\mid \pi(a\mid s)>0\}, \qquad 
	\mathcal{T}(s) &:=\{a\in A\mid Q^*(s,a)=v^*(s)\}.
\end{align*}
\end{Definition}

Next, we define the set of (in)active state-action pairs.

\begin{Definition}[Active State-Action Pairs]\label{active}
We say that a state-action pair $(s,a)\in S\times A$ is \emph{active} if $Q^*(s,a)=v^*(s)$. Otherwise, $(s,a)\in S\times A$ is  an  \emph{inactive} pair.
Let $\mathcal{I}\subseteq S\times A$ be the set of  inactive state-action pairs, and $S\times A\setminus \mathcal{I}$ thus corresponds to that of the \emph{active} ones.
\end{Definition}

We then have the following lemma which characterizes the optimal policy in terms of the inactive set $\mathcal{I}$, whose proof can be found in \S\ref{append:details_case_2}.  

\begin{Lemma}
\label{comple}
If $\pi_0$ is an optimal policy, then $\theta_{\pi_0,\mu}(s,a)=0$ for any $(s,a)\in \mathcal{I}$.
If $\pi_0(a\mid s)=0$ for any $(s,a)\in \mathcal{I}$, then $\pi_0$ is an optimal policy.
\end{Lemma}

We also introduce the following definitions for convenience of introducing the data coverage assumption.  

\begin{Definition}[Data Coverage] \label{cover}
We say that $\pi$ is a \emph{$\mu$-policy} if  $\pi(a\mid s)>0$ implies $\mu(s,a)>0$ for any $s\in S_0$. 
A \emph{$\mu$-optimal policy} is an optimal policy that is also a $\mu$-policy.
Suppose there exists at least one $\mu$-optimal policy, then, a policy $\pi^*$  is called a \emph{max-$\mu$-optimal policy} if it is a $\mu$-optimal policy 
that satisfies $|S_{\pi^*}(s)|=|S_{\pi_{\mu}}(s)\cap \mathcal{T}(s)|$ for any $s\in S_0$. 
\end{Definition}


\begin{Remark}
A $\mu$-policy means that this policy is covered by the behavior policy  in some sense.
For any state $s\in S_0$, it is reasonable to assume that at least an active pair $(s,a)$ can be visited by the behavior policy with positive probability, where $a$ is an optimal action that maximizes $Q^*(s,a)$. 
Thus, it is reasonable to assume that a $\mu$-optimal policy exists.
If a $\mu$-optimal policy exists, then the \emph{max}-$\mu$-optimal policy must exist.
\end{Remark}

We are now ready to state a new version of the  single-policy concentrability  assumption. 

\begin{Assumption}[Single-Policy Concentrability+]\label{SPC}
There exist some max-$\mu$-optimal policy $\pi^*$, and some constant $C^*>0$ such that for all $(s,a)\in S\times A$,  $\frac{\theta_{\pi^*,\mu}(s,a)}{\mu(s,a)}\le C_{\pi^*,\mu}= C^*$. 
\end{Assumption}


\begin{Remark}

Note that Assumption \ref{SPC} is slightly stronger than the usual single-policy concentrability assumption \citep{rashidinejad2021bridging,rashidinejad2022optimal} (and our Assumption \ref{ass:f1_conc}), which assumes the coverage of an  arbitrary optimal policy. 
Assumption \ref{SPC} means that if an optimal  policy is covered by the behavior policy (i.e., it is a $\mu$-optimal policy), then its \emph{occupancy measure}  should also be   covered by the offline data distribution $\mu$. 
It is reasonable in the following sense: In practice,  the offline data distribution is usually generated from the \emph{stationary distribution} of the Markov chain under some \emph{behavior policy}, which can be obtained by rolling out some infinitely (or sufficiently)  long trajectories using the policy \citep{liu2018breaking,levine2020offline}. 
Thus, $\mu$ satisfies the fixed-point equation 
\begin{equation}\label{stationary}
\mu(s')=\sum_{s,a}\mu(s,a)P_{s,a}(s').
\end{equation}
Then, the usual single-policy concentrability with initial distribution $\mu$ 
implies Assumption \ref{SPC}. 
To see this, first, we note that the usual SPC  with initial $\mu$ implies that the covered policy $\pi^*$ is a $\mu$-optimal policy by definition. Hence, there must exist a max-$\mu$-optimal policy. 

Second, we show  that for any $\mu$-optimal policy $\pi^*$, if $\theta_{\pi^*,\mu}(s)>0$ for some  $s\in S$, then $\mu(s)>0$. 
We show it by contradiction.
Suppose that $\mu(s)=0$ but $\theta_{\pi^*,\mu}(s)>0$.
Then with positive probability, there exists some trajectory $\{s_0,s_1,\cdots,s_T\}$ with $s_T=s$ generated by the $\mu$-optimal policy $\pi^*$. Note that $\mu(s_T)=0$ implies that there must exist some $t\leq T$, such that  $\mu(s_t)=0$ (with $t=T$ being the largest one). 
Let $t_0$ be the smallest $t$ such that $\mu(s_t)=0$. Then we have $t_0>0$ since the initial distribution that generates $\theta_{\pi^*,\mu}$ is $\mu$, i.e., $s_0$ is sampled from $\mu$ and thus $\mu(s_0)>0$.   
We thus have $\mu(s_{t_0-1})>0$ by the definition of $t_0$. 
Moreover, there must exist some $a\in A$ such that $\pi^*(a\mid s_{t_0-1})>0$ and $P_{s_{t_0-1},a}(s_{t_0})>0$, 
 since we have observed the transition from $s_{t_0-1}$ to $s_{t_0}$.
By the definition of $\mu$-policy, we have $\mu(s_{t_0-1},a)>0$ because $\pi^*(a\mid s_{t_0-1})>0$ for this $a$. 
We thus have $\mu(s_{t_0})>0$ by \eqref{stationary}, which contradicts the assumption.
Hence we have shown that $\theta_{\pi^*,\mu}(s)>0$ can imply $\mu(s)>0$.

Third, the second point also implies that for any $\mu$-optimal policy $\pi^*$ (including the max-$\mu$-optimal policy), $\theta_{\pi^*,\mu}(s)=0$ for any $s\notin S_0$. 
Consequently,  $\theta_{\pi^*,\mu}(s,a)>0$ implies $s\in S_0$ and $\pi^*(a\mid s)>0$, which further implies that  $\mu(s,a)>0$ since $\pi^*$ is a $\mu$-policy. Combining these three points, we obtain  our Assumption \ref{SPC}.

Finally, we note that if the $\mu$-optimal policy is {\it unique},
 Assumption \ref{SPC} is reduced to the specific  single-policy concentrability assumption used in \cite{chen2022offline}. 
In particular, \cite{chen2022offline} directly assumes that the optimal policy of the original problem is unique.

\end{Remark}

\subsection{Main Results}




We first have the following property of the max-$\mu$-optimal policy, whose proof can be found in \S\ref{append:details_case_2}. 

\begin{Proposition}\label{max_policy}
Let $\pi^*$ be a max-$\mu$-optimal policy  for which Assumption \ref{SPC} holds. Then, there exist constants $C^*,C_{\max}>0$ such that: 
\begin{enumerate} 
	\item $\theta_{\pi^*}(s,a)\le C^*\mu(s,a)$ for any $(s,a)\in S\times A$;
	\item For any $\mu$-optimal policy $\tilde\pi^*$, we have $\theta_{\tilde\pi^*}(s)\le C_{\max}\mu(s)$ for any  $s\in S$.
\end{enumerate}  
%
\end{Proposition}

{Before moving to our main theorem, we define the {\it gap} of the optimal $Q$-function, which is the minimal difference of the optimal $Q$-value between the \emph{optimal} and the \emph{second optimal}  actions,  among all  $s\in S$.

\begin{Definition}[Gap]
For each $(s,a)\in S\times A$, we define the \emph{gap} $\Delta_Q(s,a):=v^*(s)-Q^*(s,a)$. We then define the \emph{minimal}  gap as
$\Delta_Q:=\min_{(s,a)\in\mathcal{I}}~\Delta_Q(s,a),$
where we recall that $\mathcal{I}$ is the set of \emph{inactive} state-action pairs given in Definition \ref{active}.
\end{Definition}

Note that as long as  $\mathcal{I}$ is not empty, then $\Delta_Q(s,a)>0$ for any  $(s,a)\in\mathcal{I}$, leading to $\Delta_Q>0$ by definition.  If $\mathcal{I}$ is empty, then the problem becomes degenerate,    since any action is active for any state, i.e., any policy is an optimal policy. We hereafter focus on the non-degenerate   case where $\Delta_Q>0$. 
This gap notion was also used in \cite{chen2022offline} in the context of offline RL. 
However, in contrast to this work,  our definition here does not need to assume that the maximizer of $\max_{a}~Q^*(s,a)$ is {\it unique} for each $s$, which is more standard as in the online RL setting \citep{simchowitz2019non,lattimore2020bandit,papini2021reinforcement,yang2021q}. Also, our algorithm  does not need to know the gap $\Delta_Q$,  and is computationally tractable, compared to that in \cite{chen2022offline}.  Now we are ready to present   the following theorem.}

\begin{Theorem}\label{main2}
Under Assumptions \ref{realizability}, \ref{v_realizable}, \ref{SPC},  with probability at least $1-\delta$, the policy $\pi_{\Dc}$ learned from solving \eqref{alg3}  satisfies
\begin{align*}
J_{\mu}(\pi^*)-J_{\mu}(\pi_{\Dc})\le \frac{{8\sqrt{2}B_W} C_{\max}}{\Delta_Q  (1-\gamma)^3}\cdot \frac{\sqrt{\log(|W||V|/\delta)}}{\sqrt{n}}.
\end{align*}
\end{Theorem}

{Theorem \ref{main2}, with a detailed proof in \S\ref{append:details_case_2},  provides an optimal $O(1/\sqrt n)$ sample complexity under some single-policy concentrability and the realizability-only assumptions, for the return with initial distribution $\rho=\mu$. One  key in obtaining the result is a proper lower-bound constraint on $\sum_aw(s,a)\pi_{\mu}(a\mid s)$ in \eqref{alg3}. 
Note that in Theorem \ref{main2}, the value function $J_{\mu}$ corresponds to the initial distribution $\mu$.    
The next corollary connects back to the return associated with the initial distribution $\mu_0$.

\begin{Corollary}\label{cor}
Under Assumptions \ref{realizability}, \ref{v_realizable}, \ref{SPC}, and suppose that $\mu_0$ is covered by $\mu$, i.e., $\max_{s\in S}\frac{\mu_0(s)}{\mu(s)}\le C_{\mu}$ for some constant $C_{\mu}>0$,  then with probability at least $1-\delta$, the policy $\pi_{\Dc}$ learned from solving \eqref{alg3}  satisfies
\begin{align*}
J_{\mu_0}(\pi^*)-J_{\mu_0}(\pi_{\Dc})\le \frac{{8\sqrt{2}B_W} C_{\max} C_{\mu}}{\Delta_Q (1-\gamma)^3}\cdot \frac{\sqrt{\log(|W||V|/\delta)}}{\sqrt{n}}.
\end{align*}
\end{Corollary}

Corollary \ref{cor} follows by a direct change of measure argument and is thus omitted. 
\begin{Remark}

Note that Corollary \ref{cor} additionally requires the coverage of $\mu_0$ by $\mu$, which we argue is a mild assumption in the following sense: 
\begin{enumerate}
\item Recall that $S_0$ is the set of states that can be visited by $\mu$, i.e., $S_0=\{s\in S\mid \mu(s)>0\}$. This means that we only have data for states $s\in S_0$, and intuitively, it seems not  plausible to learn anything outside $S_0$ from data, without additional assumptions on the correlation among states. Therefore, we can not expect to deal with initial states outside $S_0$ and hence it is reasonable to only consider the return of $J_{\mu_0}$ with initial distribution $\mu_0$ that is  covered by $\mu$;
\item The commonly assumed single-policy concentrability  in \cite{zhan2022offline,chen2022offline,rashidinejad2021bridging,rashidinejad2022optimal} (and our Assumption \ref{ass:f1_conc}) implies that $\mu_0$ is covered by $\mu$, because 
\begin{align*}
	&\max_{s\in S}\frac{\mu_0(s)}{\mu(s)}
	=\frac{1}{1-\gamma}\cdot \max_{s\in S}\frac{(1-\gamma)\sum_{a}\mu_0(s)\pi^*(a\mid s)}{\sum_{a}\mu(s,a)}\leq \frac{1}{1-\gamma}\cdot \max_{s\in S}\frac{\sum_{a}\theta_{\pi^*,\mu_0}(s,a)}{\sum_{a}\mu(s,a)}\\
	&\qquad\leq \frac{1}{1-\gamma}\cdot \max_{s\in S,a\in A}\frac{\theta_{\pi^*,\mu_0}(s,a)}{\mu(s,a)}\leq \frac{C_{\pi^*,\mu_0}}{1-\gamma}=:C_{\mu}. 
\end{align*}
\item As stated in \citep{liu2018breaking,tang2019doubly,levine2020offline,zhan2022offline}, $\mu$ usually can be viewed as a valid occupancy measure under some behavior policy of $\pi_{\mu}$,  starting from $\mu_0$. In this case,  $C_{\mu}$ exists and satisfies $C_{\mu}\le 1/(1-\gamma)$. 
\end{enumerate}

Compared to (to our knowledge) the only \emph{gap-dependent offline} RL result 
\cite{chen2022offline}, they    require $\arg\max_a Q^*(s,a)$ to be unique for any $s$, and the algorithm is not computationally tractable.  Also, note that our algorithm does not require the knowledge of the gap $\Delta_Q$.
Compared to \cite{zhan2022offline}, we only need some  single-policy concentrability assumption for the {\it original} problem, instead of the regularized problem, together with only  the realizability assumption on the function classes. 
Moreover, our sample complexity is ${O}(1/\epsilon^2)$ with a gap dependence, while that  in \cite{zhan2022offline} is ${O}(1/\epsilon^6)$. 

Note that  \cite{zhan2022offline} also considered the vanilla version of the minimax formulation without regularization (see Section 4.5 therein). However, their 
analysis requires \emph{all-policy-concentrability} assumption, which is stronger than our assumption that only requires to cover \emph{some} single optimal policy. 
Finally, compared with the independent  work \cite{rashidinejad2022optimal} (and also our results in \S\ref{sec:func_1}), which also achieved ${O}(1/\epsilon^2)$ sample complexity, our result here does not rely on any completeness-type assumption, but is gap-dependent. 

\end{Remark}


\subsection{Proof Idea: Primal-gap-based Error Bound}\label{sec:proof_sketch_case_2}

The key technical  idea for proving the main result of Theorem \ref{main2} is a  {\it primal-gap}-based  analysis developed in  \cite{ozdaglar2022good}, which was shown to be critical in characterizing the \emph{generalization}  behaviors of  stochastic minimax optimization. Here we illustrate a proof sketch, with the missing details deferred to \S\ref{append:details_case_2}. 

To this end, we first 
recall the definitions of $\ell$ and $\ell_\Dc$ in \eqref{equ:def_ell_pop} and \eqref{equ:def_ell_emp}, respectively, and the fact that {we set $\rho=\mu$, we have} $\ell(w,v)=-u^\top w+v^\top(Kw-(1-\gamma)\mu)$ and $\ell_{\Dc}(w,v)=-u_{\Dc}^\top w+v^\top(K_{\Dc}w-(1-\gamma)\mu_{\Dc})$.
The population and empirical primal gaps are defined as follows.

\begin{Definition}[Primal Gap]
Let $\ell^V(w)=\max_{v\in V}\ell(w,v)$ and $\ell^V_{\Dc}(w)=\max_{v\in V}\ell_{\Dc}(w,v)$.
The \emph{empirical primal gap} is defined as $\Delta_{\Dc}^{W,V}(w)=\ell_{\Dc}^V(w)-\min_{w'\in W}\ell_{\Dc}^V(w'),$ and the \emph{population primal gap} is defined as $\Delta^{W,V}(w)=\ell^V(w)-\min_{w'\in W}\ell^V(w')$.
For notational simplicity, we omit the superscripts $W,V$ hereafter.
\end{Definition}

Recall that $w_{\Dc}$ is the solution to problem  \eqref{alg3}. We have $\Delta_{\Dc}(w_{\Dc})=0$.
We can upper bound the population primal gap at $w_\Dc$ as follows:


\begin{Lemma}\label{pop_pg}
Suppose Assumptions \ref{realizability}, \ref{v_realizable} hold.  
Then, with probability at least $1-\delta$, we have 
\begin{align*}
\Delta(w_{\Dc})\le \frac{{4\sqrt{2}B_W}\sqrt{\log(|V||W|/\delta)}}{(1-\gamma)\sqrt{n}}.
\end{align*}
\end{Lemma}

Next, we need to relate the primal gap to  the accuracy of policy $\pi_{\Dc}$
in terms of $J_\mu(\pi^*)-J_\mu(\pi_{\Dc})$.
Notice from Lemma \ref{comple} that, the sub-optimality gap of $\pi_{\Dc}$ can be  captured by the violation of $\pi_{\Dc}(a\mid s)=0$ for $(s,a)\in \mathcal{I}$. 
We thus have the following bound. 

\begin{Lemma}\label{inactive}
Let $\theta_{\Dc}(s,a):=w_{\Dc}(s,a)\mu(s,a)$ for all $(s,a)\in S\times A$. 
We have $\sum_{(s,a)\in \mathcal{I}} {\theta_{\Dc}(s,a)} \le \frac{\Delta(w_{\Dc})}{\Delta_Q}.$
\end{Lemma}

Finally, combining the lower bound constraints for $w$ in problem  \eqref{alg3:pop:fa}, we have the following estimate of $J_{\mu}(\pi^*)-J_{\mu}(\pi_{\Dc})$.
\begin{Lemma}\label{J_Delta}
Suppose Assumption \ref{SPC} holds. 
We have  
\begin{align*}
J_\mu(\pi^*)-J_\mu(\pi_{\Dc})\le \frac{2 C_{\max}}{(1-\gamma)^2 \Delta_Q}\Delta(w_{\Dc}).
\end{align*}
\end{Lemma}
Combining Lemma \ref{pop_pg} and Lemma \ref{J_Delta}, Theorem \ref{main2} follows.



\section{Other Extensions Under the Unified Framework}\label{sec:extensions}

We now extend our properly-constrained-LP framework in \S\ref{sec:func_1} and \S\ref{sec:second_FA} to several other  settings, showing the power 
of this framework for offline RL.  

\subsection{Unknown Behavior Policy}
In the previous results, we assume that the behavior policy $\pi_{\mu}$ is known, as in \cite{zhan2022offline,rashidinejad2022optimal}. 
If the behavior  policy is unknown, it is not easy to attain a policy corresponding to $w_{\Dc}$, as the actual the actual offline data distribution $\mu$ is usually not known.  To  tackle this issue, one can use behavior cloning as in \cite{zhan2022offline}.  
We omit the details and the reader can refer to \cite{zhan2022offline} for more related discussions.
Instead of this standard behavior-cloning-based approach, 
we propose a method that can attain the same accuracy  as solving  Programs \eqref{alg2} and \eqref{alg3}, readily adapted from our constrained-LP framework. 
The approach is based on a simple idea: if $\|w_{\Dc}(s,\cdot)\|_0=1$ for any $s$, then we can compute $\pi_{w_{\Dc}}(a\mid s)$ without knowing $\pi_{\mu}$.  Here $\|\cdot\|_0$ denotes the number of non-zero elements of a vector. 
Concretely speaking, locally in this subsection, we will choose  $\pi_{\Dc}=\pi_{w_{\Dc}}$ with $\pi_w$ defined as: 
\begin{align}\label{equ:def_pi_w_unknown}
	\pi_w(a\mid s):=\begin{cases}
     1, & \text{if~~~}{w(s,a)\neq 0}\\
     0, & \text{if~~~}{w(s,a)=0}
    \end{cases},
\end{align}
as we will constrain $\|w_{\Dc}(s,\cdot)\|_0$ later. 


\subsubsection{Completeness case: Modification of program \eqref{alg2}}

We first focus on the case that allows completeness-type assumptions, mirroring our \S\ref{sec:func_1}. 
We make a slight change to Assumption \ref{ass:f1_fake}. We assume that $W$ realizes a \emph{deterministic} optimal policy (that is covered by the  offline data), which we know always exists \cite{puterman1994markov}, instead of an arbitrary optimal policy.

\begin{Assumption}
\label{ass:f1_fake_prime}
Let $x_w:=\phi(w)$ with $\phi$ given in Definition \ref{def:x_sign}. 
Let $W$ and $B$ be the function classes for $w$ and $x_w$, respectively.  
Then, we have realizability of $W$ for a $w^*$ corresponding to a deterministic optimal policy, and $(W,B)$-completeness under $\phi$, 
  i.e., $w^*\in W$ for the optimal $w^*$ such that 
$\pi_{w^*}=\pi^*$ is a deterministic policy that satisfies Assumption \ref{ass:f1_conc},  and  $x_w\in B$ for all $w\in W$. 
  Furthermore, we assume that $W$ and $B$ are bounded, i.e., $w\geq 0$ and $\| w \|_\infty \leq B_W$ for all $ w \in W$ and $\| x \|_\infty \leq 1$ for all $x \in B$.
\end{Assumption} 

Then we add an $L_0$ constraint to Program \eqref{alg2} as follows: 
\begin{equation}
	\begin{array}{ll}
&\max_{w\in W}~~~~u_{\Dc}^\top w \qquad ~
\text{s.t.}~~~~x^\top(K_{\Dc}w-(1-\gamma)\mu_0)\le E_{n,\delta}, \quad \forall x \in B,
~~~~~\|w(s,\cdot)\|_0\le 1,~~~~~ \forall s\in S.
\label{alg2'}
\end{array}
\end{equation} 
With a slight abuse of notation, let the solution to  the above problem  be $w_{\Dc}$. We can then compute $\pi_{\Dc}=\pi_{w_{\Dc}}$ as Eq. \eqref{equ:def_pi_w_unknown}. 
Let $\bar{W}=W\cap \{w\mid \|w(s,\cdot)\|_0\le 1,\forall s\}$. Then we have the following theorem:
\begin{Theorem}
\label{main1'}
Suppose Assumptions \ref{ass:f1_conc} and  \ref{ass:f1_fake_prime} hold. Then,  with probability at least $1-6\delta$, the policy $\pi_{\Dc}$ learned from solving \eqref{alg2'} satisfies
\begin{align*}
J_{\mu_0}(\pi^*)-J_{\mu_0}(\pi_{\Dc})\le \frac{2\sqrt{2}B_W\sqrt{\log (|B||\bar{W}|/\delta)}}{(1-\gamma)\sqrt{n}}.
\end{align*}
\end{Theorem}
The proof is almost identical to that of Theorem \ref{main1} and is thus omitted here. In particular, from Lemma \ref{feasibility}, we know that $w^*$ is feasible to \eqref{alg2'}  with high probability, and by Assumption \ref{ass:f1_fake_prime}, it also satisfies $\|w(s,\cdot)\|_0\le 1$ for all $s\in S$. The rest of the proof follows from that of  Theorem \ref{main1}.

\subsubsection{Realizability-only case: Modification of program \eqref{alg3}}
We then consider the case with the realizability-only assumption, but with stronger data coverage assumption and gap-dependence, mirroring our \S\ref{sec:second_FA}.  
Let $W$ satisfy
\begin{equation}\label{equ:def_W'}
W:=\Big\{w~\Big\vert~ w\in \mathbb{R}_+^m,~\|w(s,\cdot)\|_0\le 1,~\max_{a\in A}w(s,a)\ge (1-\gamma),~\text{for~all~}s\in S\Big\}.	
\end{equation} 
Then we solve the following minimax problem using $W$  defined in \eqref{equ:def_W'} 
that satisfies Assumption \ref{realizability'}:  

\begin{equation}\label{alg3':pop:fa}
\min_{w\in W}\max_{v\in V}~~~~ \ -u_{\Dc}^\top w+v^\top(K_{\Dc}w-(1-\gamma){\mu}_{\Dc}).  
\end{equation} 
Suppose $w_{\Dc}$ is a solution to \eqref{alg3':pop:fa}. Letting $\pi_{\Dc}=\pi_{w_{\Dc}}$ as per Eq. \eqref{equ:def_pi_w_unknown}. 
Now, we slightly change Assumption \ref{realizability} such that the function class  contains at least one deterministic optimal policy.
\begin{Assumption}[Realizability and Boundedness of $W$] \label{realizability'}
There exists some solution $w^*\in W$ corresponding to a deterministic optimal policy $\pi_{w^*}$ solving  \eqref{alg3:pop:fa} with $W$ defined in \eqref{equ:def_W'} and hence  also solving \eqref{prob:minimax_pop} with $\rho=\mu$. 
Moreover, we suppose  $w\geq 0$ and $\|w\|_\infty\le B_W$ for all $w\in W$.
\end{Assumption} 

Then,  
 we have the following result:
\begin{Theorem}\label{main2_+}
Under Assumptions \ref{v_realizable}, \ref{SPC}, and \ref{realizability'}, with probability at least  $1-\delta$, the policy $\pi_{\Dc}$ learned from solving \eqref{alg3':pop:fa} satisfies 
\begin{align*}
J_{\mu}(\pi^*)-J_{\mu}(\pi_{\Dc})\le \frac{{8\sqrt{2}B_W} C_{\max}}{\Delta_Q  (1-\gamma)^3}\cdot \frac{\sqrt{\log(|W||V|/\delta)}}{\sqrt{n}}.
\end{align*}
\end{Theorem}

The proof of Theorem \ref{main2} directly applies to proving Theorem \ref{main2_+} and is thus omitted here. 

\begin{Remark}
Note that for both 
approaches above, one does not need to know the behavior policy $\pi_{\mu}$. However, it can be harder to solve since both \eqref{alg2'} and \eqref{alg3':pop:fa} involve solving mixed-integer programming problems. 
\end{Remark}

\subsection{No-coverage of Optimal Policy}\label{sec:extension_best_covered_policy}

Our algorithm and analyses can also be readily  adapted to the case where the function class does not cover the optimal $w^*$, and the optimal policy is not well-covered by the data (with almost the same proof). We provide more details below. 

\subsubsection{Completeness case}\label{sec:extension_best_covered_policy_case_1}
 
Again, we first focus on the case that allows completeness-type assumptions, as in  \S\ref{sec:func_1}. 
We start by introducing the  definition of \emph{$C$-covered policy} as in \cite{zhan2022offline}. 

\begin{Definition}[$C$-Covered Policy]\label{def:c_cover_policy}
We say a policy $\pi$ is \emph{$C$-covered} by the offline data if 
$$\theta_{\pi,\mu_0}(s,a)\le C\cdot \mu(s,a),~~\forall (s,a)\in S\times A.$$ Further, we say $\pi^*_C$ is an \emph{optimal $C$-covered policy} if it maximizes $J_{\mu_0}(\pi)$ among all $C$-covered policies. 
\end{Definition}

By this definition, we propose to solve the following optimization problem:  
\begin{align} \label{equ:pop_c_covered}
\min_{w\in W}~~~~(-u^\top w)\qquad \text{s.t.}~~~~~~Kw=(1-\gamma)\mu_0,~~\|w\|_\infty\le C.
\end{align} 
Let $w_W^*$ denote the optimal solution to this problem. Note that the solution  $w_W^*$ satisfies the validity constraints and thus $w_W^*(s,a) \mu(s,a)$ corresponds to the $\theta_{\pi^*_C}$ for some optimal $C$-covered policy as it maximizes $u^\top w$. 
Hence, with only offline datasets, we propose to solve the following empirical version of the problem  to obtain $w_{\mathcal{D}}$:
\begin{align}\label{equ:best_cover_complete_type}
\min_{w\in W}~~~~(-u_{\mathcal{D}}^\top w)\qquad\quad \text{s.t.}~~~~ x^\top(K_{\mathcal{D}}w-(1-\gamma)\mu_0)\le E_{\delta,n,W},~~~\forall x\in B,~~~~~\|w\|_{\infty}\le C,
\end{align}
where $E_{\delta,n,W}=C\sqrt{2\log(|B||W|/\delta)}/\sqrt{n}$. 
Finally, we let $\pi_{\mathcal{D}}=\pi_{w_{\mathcal{D}}}$ as before. 

We then assume the realizability of the optimal $C$-covered policy in the following assumption:
 
\begin{Assumption}[Realizability of Optimal $C$-Covered Policy]\label{assum:realizable_C_cover_policy}
The solution to Problem \eqref{equ:pop_c_covered}, denoted by $w_{W}^*$, is realizable by ${W}$, i.e., $w_{W}^*\in {W}$. 
\end{Assumption}

Note that here we do not assume the realizability of the $w^*$, but only that of $w_{W}^*$, which corresponds to the optimal $C$-covered policy. 

\begin{Theorem}
Suppose Assumptions \ref{ass:f1_fake} (without $w^*$ realizability) and \ref{assum:realizable_C_cover_policy} hold. 
With probability at least $1-\delta$, the policy $\pi_{\Dc}$ learned from solving \eqref{equ:best_cover_complete_type} satisfies 
$$J_{\mu_0}(\pi_C^*)-J_{\mu_0}(\pi_{\mathcal{D}})\le \tilde{O}\big({C\sqrt{\log(|B||W|/\delta)}}/((1-\gamma)\sqrt{n})\big).$$
\end{Theorem}

The proof is similar to that of the main results in \S\ref{sec:func_1}  (c.f. Theorem \ref{main1} and Theorem \ref{tabular_main}). 
First, we prove the  feasibility of $w_W^*$ to this empirical problem (with high probability) by using concentration bound and union bound for all $x$ in $B$. 
Next, using the feasibility and concentration bound, we have
$$u_{\Dc}^\top w_{\mathcal{D}}\ge u^\top w_W^*-\tilde{O}(1/\sqrt{n}).$$
Then, using the error bound result  (Lemma \ref{lemma:infeasible}), we have $|J_{\mu_0}(\pi_{\mathcal{D}})-u^\top w_{\mathcal{D}}|\le \|Kw_{\mathcal{D}}-(1-\gamma)\mu_0\|_1/(1-\gamma)$.
Combining the above, we have the desired result.

\subsubsection{Realizability-only case} 

We then consider the case with the realizability-only assumption, 
mirroring our \S\ref{sec:second_FA}.  
For convenience, we also assume that the optimal policy is unique,  as in \cite{chen2022offline}. 
We will consider the case with general $\mu_0$ (not necessarily equal to $\mu$). 
Following our algorithm in \S\ref{sec:second_FA} (c.f. Eq. \eqref{alg3:pop:fa}), we consider the following minimax problem: 
\begin{align}\label{equ:optimal_cover_case_realizability_only}
\min_{w\in W}\max_{v\in V}~~~~-u^\top w+v^\top(Kw-(1-\gamma)\mu_0)\qquad \mathrm{s.t.}~~~~~~~\|w\|_{\infty}\le C,~~ \sum_aw(s,a)\pi_{\mu}(a\mid s)\ge 1-\gamma. 
\end{align}
With a slight abuse of notation, We denote the optimal solution to Problem \eqref{equ:optimal_cover_case_realizability_only} as $w_W^*$. Then,  
we obtain $w_{\mathcal{D}}$ by solving the following empirical version of the problem (mirroring  Eq. \eqref{alg3}), and denote its optimal solution as $w_{\mathcal{D}}$:
\begin{align}\label{equ:optimal_cover_case_realizability_only_empirical}
\min_{w\in W}\max_{v\in V}~~~~~-u_{\mathcal{D}}^\top w+v^\top(K_{\mathcal{D}}w-(1-\gamma)\mu_{\mathcal{D}})\qquad\mathrm{s.t.}~~~~\|w\|_{\infty}\le C,~~ \sum_aw(s,a)\pi_{\mu}(a\mid s)\ge 1-\gamma.
\end{align}

We then make the following assumption regarding the sub-optimality of the solution $w_{W}^*$. 

\begin{Assumption}\label{assum:context_bandit_sub_opt}
The solution to \eqref{equ:optimal_cover_case_realizability_only},  
$w_W^*$, satisfies that  
\begin{align}\label{equ:w_w_bounded}
\max\left\{\sum_{s\in S}\sum_{a\in A}\big|w_W^*(s,a)-w^*(s,a)\big|,\sum_{s\in S}\Bigg|\frac{\theta_{\pi_{w_W^*},\mu_0}(s)}{\mu(s)}-\frac{\theta_{\pi_{w^*},\mu_0}(s)}{\mu(s)}\Bigg|\right\}\le \epsilon
\end{align}
for an   optimal $w^*$ (to Eq. \eqref{alg3:pop:fa}). Moreover, $\pi_{w_W^*}$ is a $C$-covered policy (c.f. Definition \ref{def:c_cover_policy}). 
\end{Assumption}
 
First part of this assumption characterizes how far away the optimal solution of the optimization problem above, $w_W^*$, is from $w^*$. Note that under the primal-dual formulation, the \emph{occupancy-measure constraints} are not enforced on $w$,  $w_{W}^*(s,a)\mu(s,a)$ may \emph{not}  be equal to the occupancy measure under $\pi_{w_{W}^*}$, i.e., $\theta_{\pi_{w_W^*},\mu_0}(s,a)$, which motivates the assumption on the second term  in \eqref{equ:w_w_bounded} to be controlled, in order to control the misspecification.  
Second part of the assumption assumes that optimal policy recovered from the solution $w_{W}^*$ (when there are infinitely many data) should be covered by the offline data. 
 Under this Assumption \ref{assum:context_bandit_sub_opt}, we can directly compare $w_{\mathcal{D}}$ to the optimal $w^*$ and also compare the associated policies, resulting bound that has an additional term depending on $\epsilon$.  
Concretely speaking, we have the following theorem:

\begin{Theorem}
Suppose  $w\geq 0$ and $\|w\|_\infty\le B_W$ for all $w\in W$, and the optimal policy $\pi^*$ is unique.   Suppose Assumptions 
\ref{v_realizable}  
and \ref{assum:context_bandit_sub_opt} hold. 
With probability at least $1-\delta$,  the policy $\pi_{\Dc}$ learned from solving \eqref{equ:optimal_cover_case_realizability_only_empirical}  satisfies 
$$
J_{\mu_0}(\pi^*)-J_{\mu_0}(\pi_{\mathcal{D}})\le 
\frac{4(C+\epsilon)}{(1-\gamma)^3\Delta_Q}\cdot \Big(\frac{{2\sqrt{2}C}\sqrt{\log(|V||W|/\delta)}}{\sqrt{n}}+{\epsilon}\Big).$$
\end{Theorem}
\proof{Proof Sketch.} 
The proof is similar to that of Theorem \ref{main2}, and we only provide a sketch here. 
Note that $\|u\|_{\infty}\le 1$ (where we recall $u(s,a)=r(s,a)\cdot \mu(s,a)$) and $\|v\|_\infty\le 1/(1-\gamma)$, we can easily prove that $\ell(w)$ is $\frac{2}{1-\gamma}$-Lipschitz with respect to the $\ell_1$-norm. 
Therefore, $\ell(w_W^*)-\ell(w^*)\le 2\epsilon/(1-\gamma)$ by Assumption \ref{assum:context_bandit_sub_opt}.

Using the same argument as in proving Lemma \ref{pop_pg}, 
we have $\ell(w_{\mathcal{D}})-\ell(w_W^*)\le \frac{{4\sqrt{2}C}\sqrt{\log(|V||W|/\delta)}}{(1-\gamma)\sqrt{n}}$. 
Therefore,  we have that the actual population primal-gap  $\Delta(w_{\mathcal{D}})\le \frac{{4\sqrt{2}C}\sqrt{\log(|V||W|/\delta)}}{(1-\gamma)\sqrt{n}}+\frac{2\epsilon}{1-\gamma}$. Then by Lemma \ref{inactive}, we have 
\begin{align}\label{equ:E_I_bnd}
E_I:=\sum_{(s,a)\in \mathcal{I}}w_{\mathcal{D}}(s,a)\mu(s,a)\le \frac{\Delta(w_{\mathcal{D}})}{\Delta_Q}.
\end{align} 

Then, one can follow the arguments in the proof of Lemma \ref{J_Delta}. 
Note that $\mathcal{T}_{\mu}(s)$ in the proof may be empty in this setting, while under the uniqueness assumption of $\pi^*$, $|\mathcal{T}(s)|=1$ for all $s$. Hence, we will replace the $\mathcal{T}_{\mu}(s)$ in the proof by $\mathcal{T}(s)$, i.e., for all $s\in S_0$,  for $a\in\mathcal{T}(s)$, we let $\tilde\theta(s,a)=\theta_{\Dc}(s,a)+\sum_{a':(s,a')\in\mathcal{I}}\theta_{\Dc}(s,a')$; for any other $a\in A$, we let $\tilde\theta(s,a)=0$. Note that this construction yields $\tilde \pi^*=\pi^*$ under a unique $\pi^*$. The rest of the proof follows, with  the following inequalities:
\begin{align}
&J_{\mu_0}(\pi^*)-J_{\mu_0}(\pi_{\mathcal{D}})\le \frac{2E_I}{1-\gamma}\cdot \max_{s\in S}\left(\frac{\theta_{\pi^*,\mu_0}(s)/\mu(s)}{\sum_aw_{\mathcal{D}}(s,a)\pi_{\mu}(a\mid s)}\right)
\nonumber\\
&\quad\le\frac{2E_I}{1-\gamma}\max_{s\in S}\frac{\theta_{\pi_{w_W^*},\mu_0}(s)/\mu(s)}{1-\gamma}+\frac{2E_I}{1-\gamma}\cdot \frac{\epsilon}{1-\gamma}\label{equ:tmp_1}\\
&\quad\leq \frac{2CE_I}{(1-\gamma)^2}+\frac{2E_I\epsilon}{(1-\gamma)^2},
\label{equ:tmp_2}
\end{align}
where 
Eq. \eqref{equ:tmp_1} uses Assumption \ref{assum:context_bandit_sub_opt} and the 
lower bound constraints of $\sum_aw(s,a)\pi_{\mu}(a\mid s)$. 
Plugging in $\Delta(w_{\mathcal{D}})\le \frac{{4\sqrt{2}C}\sqrt{\log(|V||W|/\delta)}}{(1-\gamma)\sqrt{n}}+\frac{2\epsilon}{1-\gamma}$ and Eq. \eqref{equ:E_I_bnd},  we have the desired result. 
\hfill \moronly{$\square$}
\endproof

\subsection{Contextual Bandit without Completeness}

We now consider a simplified setting than the general MDP setting, i.e., the contextual bandit setting, where there state is drawn from a fixed distribution. Recall the formal introduction in \S\ref{sec:model_setup}.

We show below that our algorithm and  results can be readily  specialized to this case,  without the completeness assumption (c.f. Assumption \ref{ass:f1_fake} in \S\ref{sec:func_1}).  
 


\subsubsection{Results}
As stated in \S\ref{sec:background}, contextual bandit is a special case of learning in a discounted MDP with $\gamma=0$.
We now introduce the corresponding linear-programming formulation: 
\begin{align}\label{equ:cb_pop}
\min_{w\ge 0}~~~~(-u^\top w)\qquad\quad \text{s.t.}~~~~~~\Omega w=\mu_0, ~~~~w\in W,
\end{align}
where $\Omega=\mathrm{Diag}(\Omega_1,\cdots,\Omega_{|S|})$ with $\Omega_i=(\mu(s^i,a^1),\cdots,\mu(s^i,a^{|A|}))$.
We propose to use function approximation to $w$ and $\mu$ here, i.e., $w\in W$ and $\mu\in U$ for some function classes $W$ and $U$. First, we use $\hat{\mu}_{D}$ to estimate $\mu$  by maximum likelihood estimation (MLE):
$$\hat{\mu}_{\mathcal{D}}\in\arg\max_{\mu'\in U}~~\mathbb{E}_{(s,a)\sim \mu_{\mathcal{D}}}\log \mu'(s,a).$$
With this MLE estimator, the algorithm still falls into our LP framework, though it is slightly different from that for the MDP case. Note that the contextual bandit setting was also considered in \cite{rashidinejad2022optimal}, which is also  related but different from that for the MDP case.


By standard MLE sample complexity   results (see e.g., \cite[Chapter 7]{geer2000empirical}), we  have  the following lemma under the following realizability assumption:

\begin{Assumption}[Realizability of $U$]\label{assum:mu_realizable}
	Suppose $\mu\in U$ and $|U|<\infty$. 
\end{Assumption}

\begin{Lemma}\label{mle}
Under Assumption \ref{assum:mu_realizable}, with probability at least $1-\delta$, we have 
$$\|\hat{\mu}_{\mathcal{D}}-\mu\|_1\le O(\sqrt{\log (|U|/\delta)}/\sqrt{n}).$$
\end{Lemma}
We next solve the following empirical version of \eqref{equ:cb_pop}:
\begin{align}\label{equ:cb_empirical}
\min_{w\ge 0}~~~~(-u_{\mathcal{D}}^\top w)\qquad\quad \text{s.t.}~~~~~~\|\Omega_{\mathcal{D}} w-\mu_0\|_1\le E_{\delta,n},~~ w\in W,
\end{align}
where $E_{\delta,n}=6\sqrt{\log(|U|/\delta)}/\sqrt{n}$ and $\Omega_{\mathcal{D}}=\mathrm{Diag}(\Omega_{1,D},\cdots, \Omega_{|S|,D})$ with 
$$
\Omega_{i,D}=(\hat{\mu}_{\mathcal{D}}(s^i,a^1),\cdots, \hat \mu_{D}(s^i,a^{|A|})).
$$ 
Let $w_{\Dc}$ denote the solution to \eqref{equ:cb_empirical}, then we let $\pi_{\Dc}=\pi_{w_{\Dc}}$. 

This Lemma \ref{mle} plays the same role as Eq. \eqref{single_concentration}, in the contextual-bandit case. This lemma is used to prove that $w^*$ is feasible with high probability. 
Then using similar arguments that proved Theorem \ref{main1} (importantly Lemma \ref{lemma:infeasible}), we have the following result. 

\begin{Theorem}
Suppose Assumptions \ref{ass:f1_conc} and \ref{assum:mu_realizable} hold. 
Then, with probability at least $1-\delta$,  the policy $\pi_{\Dc}$ learned from solving \eqref{equ:cb_empirical} satisfies 
$$J_{\mu_0}(\pi^*)-J_{\mu_0}(\pi_{\mathcal{D}})\le \tilde{O}(\sqrt{\log(|W||U|/\delta)}/\sqrt{n}).$$
\end{Theorem}

\begin{Remark}
Notice that in contextual bandits, we do not need to deal with the transition dynamics. Also, we propose to approximate $\mu$ in the algorithm. In this case, we do not need to use the sign function $x$ to deal with the $\ell_1$ norm violation of the constraint, which thus removes 
the completeness assumption of the corresponding function class $B$. 
\end{Remark}



\section{Concluding Remarks}\label{sec:conclusion}

In this paper, we revisited the linear-programming framework for  offline RL, with a focus on the case with general function approximation, as  recently  advocated in \cite{zhan2022offline} to design algorithms to handle partial data coverage with computational tractability. 
We proposed a series of  offline RL algorithms based on the core idea of adding certain \emph{error-bound induced  constraints} to the LP. Under a completeness-type assumption, we established an optimal $O(1/\sqrt{n})$ sample complexity under the standard single-policy concentrability assumption, with  simple analyses. The key technique is a novel error bound that relates the constraint violation of the occupancy-measure validity to the value function suboptimality. Such a framework was then readily instantiated to the tabular MDP case, for both discounted and average-reward settings, achieving either state-of-the-art or the first sample complexities in the  literature. To remove the completeness assumption, we developed another novel error bound that relates the suboptimality with the primal-gap in the minimax reformulation of the LP. We then 
proposed a lower-bound constraint on the
density ratio, which, under a stronger version of single-policy concentrability assumption, does not lose the optimality while stabilizing the normalization step in generating
the policy from the LP solution. This way, we established a $O(1/\sqrt{n})$ sample complexity that depends on the gap of the value function. Both constrained-LP algorithms were then extended to several other settings under  a unified framework. 


Our work has opened  up fruitful  avenues for  future research in offline RL. For example, is it possible to achieve optimal $O(1/\sqrt{n})$ sample complexity with the standard single-policy concentrability assumption and only realizability, under our constrained-LP framework? What is the gap-dependent lower bound for offline RL with general  function approximation? 
Can optimal $O(1/\sqrt{n})$ rate be achieved for the average-reward setting with a weaker assumption of 
bounded-span bias functions? 
Can weaker data-coverage metrics, e.g., other notions of concentrabilities as in \cite{zhu2024importance}, reconcile with the constrained-LP framework while maintaining the $O(1/\sqrt{n})$ rate? 
We hope our results can provide some insights into addressing these questions. 


\section*{Acknowledgement}
{We would like to express our sincere gratitude to the anonymous reviewers of  International Conference on Machine Learning (ICML) 2023 for the valuable feedback. S.P. acknowledges support from MathWorks Engineering Fellowship. A.O and K.Z. were supported by MIT-DSTA grant 031017-00016. K.Z. also acknowledges  support  from Simons-Berkeley Research Fellowship and Army Research Laboratory (ARL) Grant W911NF-24-1-0085.}


\appendix

\section{Omitted Details in \S\ref{sec:func_1}}\label{append:func_1}

\subsection{Proof of Lemma \ref{lemma:infeasible}}

Let the policy that is obtained by normalizing both $\theta$ and $\theta_{\pi_{\theta}}$ be $\pi_\theta$ (note that normalizing both of these vectors gives the same policy). Next, we define:
\begin{align}\label{equ:def_bar_theta}
\bar{\theta}(s) = \sum_{a\in A} \theta(s,a), \qquad \text{and} \qquad \bar{\theta}_{\pi_\theta}(s) = \sum_{a\in A} \theta_{\pi_\theta}(s,a).
\end{align}
Note that we can write $\theta(s,a) = \bar{\theta}(s)\pi_\theta(a\mid s)$ and $\theta_{\pi_\theta}(s,a) = \bar{\theta}_{\pi_\theta}(s) \pi_\theta(a\mid  s)$.

Let $P_{\pi_\theta}\in \mathbb{R}^{|S|\times |S|}$ be  a column stochastic matrix (the sum  of all entries of every column is $1$), which describes the state transition probabilities under the policy ${\pi_\theta}$, i.e.,  
\begin{align*}
P_{\pi_{\theta}}(j,i)=\sum_{a\in A}P_{s^i,a}(s^j)\cdot \pi_{\theta}(a\mid s^i).	
\end{align*}


Also, we define the matrix $G_\theta = \text{Diag}(\pi_\theta(\cdot \mid s^1), \pi_\theta(\cdot \mid  s^2), \cdots,  \pi_\theta(\cdot \mid  s^{|S|})) \in  \mathbb{R}^{|S||A|\times |S|}$, and notice the fact that $M G_\theta=I - \gamma P_{\pi_\theta}$. Now, since $\theta_{\pi_\theta}$ satisfies the constraints in Problem \eqref{P0}, we have $M\theta_{\pi_\theta} = (1 - \gamma)\mu_0$. This implies:
\begin{align}
\label{eq:lemma_one_help}
&\|M\theta-(1-\gamma)\mu_0\|_1 = \| M(\theta - \theta_{\pi_\theta}) \|_1 = \| M G_\theta (\bar{\theta} - \bar{\theta}_{\pi_\theta}) \|_1 \notag\\
&\qquad= \| (I - \gamma P_{\pi_\theta}) (\bar{\theta} - \bar{\theta}_{\pi_\theta}) \|_1 \geq (1 - \gamma) \| \bar{\theta} - \bar{\theta}_{\pi_\theta} \|_1.
\end{align}

Here  the last inequality is because  $\gamma \|P_{\pi_{\theta}}(\bar{\theta} - \bar{\theta}_{\pi_\theta})\|_1\le \gamma \|\bar{\theta} - \bar{\theta}_{\pi_\theta}\|_1$, which follows from the fact that $P_{\pi_{\theta}}$ is a column stochastic matrix.

On the other hand, since $r(s, a) \in [0, 1]$ for all $(s, a)$, we have:
\begin{align}
\label{eq:lemma_two_help}
|r^\top(\theta-\theta_{\pi_{\theta}})| = | r^\top G_\theta (\bar{\theta} - \bar{\theta}_{\pi_\theta}) | \leq \| \bar{\theta} - \bar{\theta}_{\pi_\theta} \|_1.
\end{align}
Combining inequalities \eqref{eq:lemma_one_help} and \eqref{eq:lemma_two_help}, we obtain the result.\hfill$\square$

\subsection{Proof of Theorem \ref{main1}}
First, we need to guarantee the feasibility of the optimization problem  \eqref{alg2}. 
\begin{Lemma}\label{feasibility}
Any $w^*\in W$ (see Assumption \ref{ass:f1_fake}) is feasible to \eqref{alg2} with probability at least $1-\delta$.
\end{Lemma}
\proof{Proof.} 
Use Hoeffding's inequality, we have that for any $x \in B$:
\begin{align}
\mathbb{P}(x^\top(K-K_{\Dc})w^* \geq  t) \leq \exp\left( \frac{-nt^2}{{8B_W^2}} \right).
\end{align}
This is using the fact that $x^\top(K-K_{\Dc})w^*$ is a random variable which lies in the interval $[-B_W(1+\gamma), B_W(1+\gamma)]$. {Note that we use the fact of $\|w^*\|_\infty\leq C^*\leq B_W$, by Assumption \ref{ass:f1_fake}. 
}  
Now, taking $t = 2 \sqrt{2} B_W  \frac{\sqrt{\log (|W||B|/\delta)}}{\sqrt{n}}$, we have
\begin{align}
\mathbb{P}(x^\top(K-K_{\Dc})w^* \ge  t) \leq \frac{\delta}{|W||B|}.
\end{align}
Now, taking the union bounds over all $x \in B$ and also all $w\in W$, we get the result.
\hfill \moronly{$\square$}
\endproof
 

Next, we show that the objective value $u_{\Dc}^\top w_{\Dc}$ is close to $u^\top w^*$.
\begin{Lemma}
\label{lemma:bd_diff_u_d}
We have
\begin{align*}
u_{\Dc}^\top w_{\Dc}\ge u^\top w^*- \frac{\sqrt{2} B_W}{\sqrt{n}} \sqrt{\log{\frac{1}{\delta}}}	
\end{align*}
with probability at least $ 1- 2\delta$.
\end{Lemma}
\proof{Proof.}
From Lemma \ref{feasibility}, we have
\begin{align*}
u^\top_{\Dc}w_{\Dc}\ge u_{\Dc}^\top w^*
\end{align*}
with probability at least $1 - \delta$ (since $w^*$ is feasible to \eqref{alg2} with probability $1 - \delta$).

Then we can use Hoeffding's inequality to bound $(u-u_{\Dc})^\top w^*$ as follows:
\begin{align}
\mathbb{P} ( u^\top_{\Dc}w^* \le u^\top w^* - t) \leq \exp \left( \frac{-nt^2}{2{B_W^2}} \right).
\end{align}
Setting this upper bound to be equal to $\delta$, we have:
\begin{align}
t = \frac{\sqrt{2} B_W}{\sqrt{n}} \sqrt{\log{\frac{1}{\delta}}}.
\end{align}
Combining the two events completes the proof.
\hfill \moronly{$\square$}
\endproof

Next, we provide a bound for $\|Kw_{\Dc}-(1-\gamma)\mu_0\|_1$.
\begin{Lemma}
\label{lemma:final_hlper_func1}
We have	
\begin{align*}
\|Kw_{\Dc}-(1-\gamma)\mu_0\|_1\le 2E_{n,\delta}
\end{align*}
with probability at least $1-2\delta$.
\end{Lemma}
\proof{Proof.}
We first have
\begin{align}\label{single_concentration}
x^\top(K-K_{\Dc})w \le E_{n,\delta}, \qquad \forall x \in B
\end{align}
for any $x\in B,~w\in W$ with probability at least $1-\delta$, by a concentration bound and union bound (similar to the proof of Lemma \ref{feasibility}).
This directly implies our lemma since $w_{\Dc}\in W$. 

Therefore:
\begin{align}\label{equ:separate_l1_norm}
&\|Kw_{\Dc}-(1-\gamma)\mu_0\|_1=\max_{x\in B}~x^\top \left(Kw_{\Dc}-(1-\gamma)\mu_0\right)\nonumber\\
&\quad\le \max_{x\in B}~x^\top \left( K_{\Dc}w_{\Dc} -(1-\gamma)\mu_0 \right) + \max_{x\in B}~x^\top \left(K-K_{\Dc}\right)w_{\Dc} \leq E_{n,\delta} + E_{n,\delta},
\end{align}
 where the first term on the right-hand side is due to that $w_\Dc$ satisfies the constraint in \eqref{alg2}.  
This completes the proof.
\hfill \moronly{$\square$}
\endproof

Finally, combining the above lemmas and Lemma \ref{lemma:infeasible}, we can prove Theorem \ref{main1}:
\proof{[Proof of Theorem \ref{main1}].}
From Lemma \ref{lemma:bd_diff_u_d}, we have:
\begin{align}
J_{\mu_0}(\pi^*) = u^{\top} w^* \leq  u_\Dc^\top w_\Dc + \frac{\sqrt{2} B_W}{\sqrt{n}} \sqrt{\log{\frac{1}{\delta}}}.
\end{align}
This  tells us that $u_\Dc^\top w_\Dc$ is close to $u^{\top} w^* = J_{\mu_0}(\pi^*)$. Next, using Hoeffding's inequality and union bound, similarly we know that
\begin{align}\label{objective_gen}
 u_\Dc^\top w \leq u^\top w+ \frac{\sqrt{2} B_W}{\sqrt{n}} \sqrt{\log{\frac{|W|}{\delta}}}
 \end{align}
for any $w\in W$.
Then since $w_{\Dc}\in W$, we have
 \begin{align}
 u_\Dc^\top w_\Dc \leq u^\top w_\Dc + \frac{\sqrt{2} B_W}{\sqrt{n}} \sqrt{\log{\frac{|W|}{\delta}}}
 \end{align}
 with probability at least $1 - \delta$. 
{Define $\hat{\theta}_{\Dc}(s,a) =  \tilde{\theta}_\Dc(s,a)\mu(s)$ where we recall the definition of $\tilde{\theta}_\Dc$ in \eqref{equ:def_tilde_theta}. 
Note that ${\theta}_{\pi_{\hat{\theta}_{\Dc}}}  = {\theta}_{\pi_{\tilde{\theta}_{\Dc}}} $.} 
Next, using the definition of $u$, we have $u^\top w_\Dc = r^{\top} \hat{\theta}_\Dc$. Now, using Lemma \ref{lemma:infeasible}, we can bound the difference
 \begin{align}
 r^\top \hat{\theta}_{\Dc} \leq  r^\top {\theta}_{\pi_{\tilde{\theta}_{\Dc}}} + \frac{\|M\hat{\theta}_{\Dc}-(1-\gamma)\mu_0\|_1}{1-\gamma},
 \end{align} 
 where we recall that $\pi_{\tilde{\theta}_{\Dc}}$ is generated by $\tilde{\theta}_\Dc$ by normalization.
Note that here we have $r^\top {\theta}_{\pi_{\tilde{\theta}_{\Dc}}} = J_{\mu_0}(\pi_{\tilde{\theta}_{\Dc}})=J_{\mu_0}(\pi_{\Dc}).$  Finally, using Lemma \ref{lemma:final_hlper_func1}, we can bound $\|M\hat{\theta}_{\Dc}-(1-\gamma)\mu_0\|_1 = \|Kw_{\Dc}-(1-\gamma)\mu_0\|_1$,
which completes the proof.
\hfill \moronly{$\square$}
\endproof

\subsection{Proof of Theorem \ref{tabular_main}}

Note that Theorem \ref{main1} is not directly applicable  to derive the sample complexity of this algorithm. 
Though not directly applicable, we can still base our analysis on the derivations above. 
Specifically, here we provide a proof for this theorem based on the number of extreme points of the convex sets. 
\proof{Proof.}
The number of extreme points of $W$ and $B$ are  $(|A|+1)^{|S|}$ and $2^{|S|}$.
With a slight abuse of notation, we let $\|W\|_e$ and $\|B\|_e$ denote the number of extreme points of $W$ and $B$, respectively.  

According to the proof of Theorem \ref{main1}, we only need to modify  two union concentration bounds -- \eqref{single_concentration} and \eqref{objective_gen}.
These two inequalities can be replaced by the following two inequalities in terms of the number of extreme points: 
\begin{enumerate}
\item $|x^\top(K-K_{\Dc})w| \le 2B_W\sqrt{|S|\log((2|A|+2)/\delta)}/{\sqrt{n}}$   
for any $w\in W,~x\in B$ with probability at least $1-\delta$;
\item $|(u-u_{\Dc})^\top w|\le 2B_W\sqrt{|S|\log((|A|+1)/\delta)}/\sqrt{n}$ 
for any $w\in W$ with probability at least $1-\delta$.
\end{enumerate}
\sloppy We only prove the first claim and the second one follows  similarly. Let $W_0=\{w_1,\cdots,w_{\|W\|_e}\},B_0=\{x_1,\cdots,x_{\|B\|_e}\}$ be the sets of extreme points of $W,B$, respectively. Then for any $w\in W$, we have
\begin{align*}
	w=\sum_{i=1}^{\|W\|_e}\lambda_iw_i,
\end{align*}
and for any $x\in B$, 
\begin{align*}
	x=\sum_{j=1}^{\|B\|_e}\zeta_jx_j,
\end{align*}
for some $\lambda=(\lambda_1,\cdots,\lambda_{\|W\|_e})^\top$ and $\zeta=(\zeta_1,\cdots,\zeta_{\|B\|_e})^\top$ that lie in the corresponding  simplices. 
For any $w_i,~x_j$, using Hoeffding inequality and union bound on the sets $W_0,B_0$, we have
\begin{align*}
	|x_j^\top(K-K_{\Dc})w_i|\le 2B_W\sqrt{|S|\log((2|A|+2)/\delta)}/\sqrt{n}
\end{align*}
with probability at least $1-\delta$. 
Then using the decomposition
\begin{align*}
	x^\top(K-K_{\Dc})w=\sum_{i=1}^{\|W\|_e}\sum_{j=1}^{\|B\|_e}\lambda_i\zeta_jx_j^\top(K-K_{\Dc})w_i
\end{align*}
and the Jensen's inequality, we prove that with probability at least $1-\delta$, we have 
\begin{align*}
	|x^\top(K-K_{\Dc})w|\le 2B_W\sqrt{|S|\log((2|A|+2)/\delta)}/\sqrt{n}
\end{align*} 
This completes the proof of the first claim. 
Furthermore, note that \eqref{equ:separate_l1_norm} still holds since our $B$ now includes \emph{all} the sign functions of a vector $w\in \mathbb{R}^m$, and thus $\|Kw_{\Dc}-(1-\gamma)\mu_0\|_1=\max_{x\in B}~x^\top \left(Kw_{\Dc}-(1-\gamma)\mu_0\right)$. 
Using the same strategy as the rest of Theorem \ref{main1}'s proof, we can prove Theorem \ref{tabular_main}. 
\hfill \moronly{$\square$}
\endproof

\subsection{Proof of Proposition  \ref{sparsity}} Now we provide the proof of  Proposition  \ref{sparsity}. 
\proof{Proof.} 
For any $w\in W$, we recall the definition of $S(w):=\{s~\vert~\sum_{a\in A}w(s,a)\mu(s,a)\ne 0, s\in S\}$ to be the states  such that $\sum_{a\in A}w(s,a) \mu(s,a)\ne 0$. 
Then, for any $s\ne  S(w)$, we know that for all $a\in A$,  $w(s,a)\mu(s,a)=0$. 
Recall the definition of $K$, we have $[Kw-(1-\gamma)\mu_0]_s=\sum_aw(s,a)\mu(s,a)-\gamma \sum_{\bar s,a}P(s\mid \bar s,a)w(\bar s,a)\mu(\bar s,a)-(1-\gamma)\mu_0(s)$. 
Therefore, for any $s\notin S(w)$, 
\begin{align*}
[Kw-(1-\gamma)\mu_0]_s\le 0-\gamma \sum_{\bar s,a}P(s\mid \bar s,a)w(\bar s,a)\mu(\bar s,a)-(1-\gamma)\mu_0(s)\le 0.
\end{align*}
This implies that $[\phi(w)]_s=-1,\forall s\notin S(w)$, i.e., for those $s\notin S(w)$, its sign function value is known and fixed (as $-1$). Since $|S(w)|\leq q$ for any  $w\in W$, we have $\|\phi(w)+\mathbf{1}^{|S|}\|_0\le q$. 
This further implies that for $B=\{x\in \{-1,1\}^{|S|}\mid \|x+\mathbf{1}^{|S|}\|_0\le q\}$, $(W,B)$ is complete under $\phi$, i.e., for any $w\in W$, $\phi(w)\in B$. 
The total number of elements in $B$ is bounded by $|W|\cdot 2^q$, which completes the proof. 
\hfill \moronly{$\square$}
\endproof

\newpage

\section{Omitted Details in \S\ref{sec:armdp}}\label{append:avg_mdp}

%
%


The overall proofs of Theorems \ref{main1_avg} and \ref{tabular_avg_main} follow directly from those of Theorems \ref{main1} and \ref{tabular_main}, respectively. We will only emphasize the key differences and steps here.  
We start with the following definitions. 

\begin{Definition}
For any policy $\pi$, we use $\beta_{\pi}$ to denote the \emph{stationary distribution} of the Markov chain under the  transition matrix $P_{\pi}$. 
For any $\theta\in \mathbb{R}^{m}$, we use $\beta_{\theta}\in \mathbb{R}^{|S|}$ to  denote the vector  such that   $\beta_{\theta}(s)=\sum_{a\in A}\theta(s,a)$ for all $s\in S$. 
\end{Definition}

Then, under Assumption \ref{mix}, it is direct to obtain the following lemma.

\begin{Lemma}\label{mix_later}
	Suppose  that $n\geq 32 B_W^2 \log(|B||W|/\delta)$, and let  $\theta(s,a)=w(s,a)\mu(s,a)$ for all $(s,a)\in S\times A$, then 
	there exists some $T_0>0$ such that   $\|P_{\pi_\theta}^{T_0}\beta-\beta_{\pi_\theta}\|_1\le E_{n,\delta}$ for any $w\in W$ and any state distribution $\beta\in\Delta(S)$.
\end{Lemma}

The key in proving Theorems \ref{main1} and \ref{tabular_main} was the error bound result in Lemma \ref{lemma:infeasible}, which connects the $\ell_1$-norm of the occupancy measure validity constraint violation to the suboptimality of the value. We prove the counterpart below for the average-reward case, which relates $r^\top \theta$ to $J(\pi_{\theta})$. 



\begin{Lemma}\label{eb}
\begin{enumerate}
\item If $\theta\in \Delta(S\times A)$, and $M\theta=0$, then we have
$$r^\top\theta=J(\pi_{\theta}).$$
\item For any $w\in W$, if $\theta(s,a)=w(s,a)\mu(s,a)$ for all $(s,a)\in S\times A$ and $n\geq 32 B_W^2 \log(|B||W|/\delta)$, then 
$$\|\beta_{\theta}-\beta_{\pi_{\theta}}\|_1\le T_0\|(I-P_{\pi_{\theta}})\beta_{\theta}\|_1+3T_0\bigg|\sum_{s,a}w(s,a)\mu(s,a)-1\bigg|+E_{n,\delta}.
$$
\end{enumerate} 
\end{Lemma}
\proof{Proof of Lemma \ref{eb}.}
We only prove the second part and the first part is just a special case, which can also be found in \cite{puterman1994markov}. 
First, let  $\bar{\beta}_{\theta}\in \mathbb{R}^{|S|}$ be defined as the \emph{normalized} version $\beta_{\theta}$:
$$
\bar{\beta}_{\theta}=\frac{\beta_{\theta}}{\sum_{s\in S}\beta_{\theta}(s)},
$$
where the division is element-wise. 
Then,  $\bar{\beta}_{\theta}$ satisfies  that 
\begin{enumerate}
\item $\bar{\beta}_{\theta}$  lies in the simplex: $\sum_s\bar{\beta}_{\theta}(s)=1$ and $\bar{\beta}_{\theta}(s)\geq 0$;
\item The following inequality holds: 
\begin{align}\label{simplex}
&\|\beta_{\theta}-\bar{\beta}_{\theta}\|_1=\sum_{s}\bigg|\sum_{a}{\theta}(s,a)-\frac{\sum_{a}{\theta}(s,a)}{\sum_{s',a}{\theta}(s',a)}\bigg|\notag\\
&\quad=\sum_{s}\bigg|\frac{(\sum_{s',a}{\theta}(s',a)-1)\sum_{a}{\theta}(s,a)}{\sum_{s',a}{\theta}(s',a)}\bigg|\le \frac{\big|\sum_{s',a}{\theta}(s',a)-1\big|\sum_{s,a}{\theta}(s,a)}{\sum_{s',a}{\theta}(s',a)}\notag\\
&\quad= \bigg|\sum_{s,a}w(s,a)\mu(s,a)-1\bigg|.
\end{align}
\end{enumerate}  
Next, we prove that $\|\bar{\beta}_{\theta}-\beta_{\pi_{\theta}}\|_1\le \|(I-P_{\pi_{\theta}})\bar{\beta}_{\theta}\|_1$.
Furthermore, we have 
\begin{align}
\|\beta_{\pi_{\theta}}-\bar{\beta}_{\theta}\|_1-E_{n,\delta}&\le\|\beta_{\pi_{\theta}}-\bar{\beta}_{\theta}\|_1-\|P_{\pi_\theta}^{T_0}\bar \beta_{\theta}-\beta_{\pi_{\theta}}\|_1\label{errorboundave_1}\\
&\le \|\bar{\beta}_{\theta}-P_{\pi_{\theta}}^{T_0}\bar{\beta}_{\theta}\|_1\label{errorboundave_2}\\
&= \bigg\|\left(\sum_{i=0}^{T_0-1}P_{\pi_{\theta}}^i\right)(I-P_{\pi_{\theta}})\bar{\beta}_{\theta}\bigg\|_1\label{errorboundave_3}\\
&\le T_0\|(I-P_{\pi_{\theta}})\bar{\beta}_{\theta}\|_1\label{errorboundave_4}\\
&\le  T_0\|(I-P_{\pi_{\theta}})\beta_{\theta}\|_1+2T_0\|\bar{\beta}_{\theta}-\beta_{\theta}\|_1,\label{errorboundave}
\end{align}
where Eq. \eqref{errorboundave_1} follows from Lemma \ref{mix_later}, Eq. \eqref{errorboundave_2} uses the triangular inequality, Eq. \eqref{errorboundave_3} uses the fact that $I-P_{\pi_{\theta}}^{T_0}=\left(\sum_{i=0}^{T_0-1}P_{\pi_{\theta}}^i\right)(I-P_{\pi_{\theta}})$ with $P_{\pi_{\theta}}^0=I$,  Eq. \eqref{errorboundave_4} is due to the fact that $\|P_{\pi_{\theta}}\beta\|_1\le \|\beta\|_1$ for any distribution $\beta$ and stochastic matrix $P_{\pi_{\theta}}$, and Eq. \eqref{errorboundave}  is also due to the triangular inequality and Eq. \eqref{simplex}. 
Finally, combining \eqref{errorboundave} and \eqref{simplex} yields the desired result.
\hfill \moronly{$\square$}
\endproof


The rest of the proof follows directly from that of Theorems \ref{main1}, by i)  validating the high probability \emph{feasibility} of $w^*$ (as in  Lemma \ref{feasibility}), via the concentrations of $K_\Dc$ to $K$ and of $\mu_\Dc$ to $\mu$; ii) the closeness between $u_{\Dc}^\top w_{\Dc}$ and $u^\top w^*$ (as in Lemma \ref{lemma:bd_diff_u_d}), via the concentration of $u_\Dc$ to $u$; iii) upper bound $\|Kw_{\Dc}\|_1$ and $\big|\sum_{s,a}w_{\Dc}(s,a)\mu(s,a)-1\big|$ by $O(E_{n,\delta})$ (as in Lemma \ref{lemma:final_hlper_func1}), via the concentrations of $K_{\Dc}$ to $K$ and $\mu_{\Dc}$ to $\mu$; iv) apply the error bound in Lemma \ref{eb}  to control the suboptimality gap of the learned $\pi_{w_{\Dc}}$ in terms of the objective value. 
Similar analogy to the proof of Theorem \ref{tabular_main} also applies when proving Theorem \ref{tabular_avg_main} for the tabular case, which completes the proofs for the results in this section.



\section{Omitted Details in \S\ref{sec:second_FA}}\label{append:details_case_2}

\subsection{Proof of Lemma \ref{comple}}

We prove the first part as follows.
If $\pi_0$ is optimal, $\pi_0(a\mid s)=0$ for any  inactive $(s,a)$.
Therefore, $\mathbb{P}_{\pi_0}(s^t=s,a^t=a;\mu)=0$  
for any $t$ and any $(s,a)\in \mathcal{I}$. Therefore, $\theta_{\pi_0,\mu}(s,a)=0$ for any $(s,a)\in \mathcal{I}$.

For the second part, we prove it as follows.
Let $v^*$ be the optimal value function. We only need to prove $v_{\pi_0}=v^*$.
For a policy $\pi$, define
\begin{align*}
P_{\pi}(j,i)=\sum_{a\in A}P_{s^i,a}(s^j)\cdot \pi(a\mid s^i)
\end{align*}
and
\begin{align*}
r_{\pi}=(r(s^1,\cdot)^\top \pi(s^1,\cdot),\cdots,r(s^{|S|},\cdot)^\top\pi(s^{|S|},\cdot))^\top.
\end{align*}
Then we know that $v_{\pi_0}$ is the unique  solution to the linear equation:
\begin{equation}\label{equation}
v=\gamma P^\top_{\pi_0}v+r_{\pi_0}.
\end{equation}
We then prove that $v^*$ is also a solution to this equation. 
In fact,  letting $P_{\pi}(i)$ be the $i$-th column of $P_{\pi}$, we have
\begin{align}
&\gamma P_{\pi_0}^\top(i)v^*+r_{\pi_0}(i)\\
&=\sum_{a\in S_{\pi_0}(s^i)}\pi_0(a\mid s^i)(\gamma P_{s^i,a}^\top v^*+r(s^i,a))\\
&=\sum_{a\in S_{\pi_0}(s^i)}\pi_0(a\mid s^i)Q^*(s^i,a)\\
&=v^*(s^i),
\end{align}
where the second equality is because $\gamma P^\top v^*+r=Q^*$ and the third equality is because $Q^*(s,a)=v^*(s)$ for $a\in S_{\pi_0}(s)$, as these are the action $a$s such that the $(s,a)$ is not in $\mathcal{I}$. 
Then $v^*$ is the solution to \eqref{equation}. Since the solution to \eqref{equation} is unique \citep{puterman1994markov}, we have $v^*=v_{\pi_0}$, which yields the desired result.\hfill$\square$

\subsection{Proof of Proposition \ref{max_policy}}\label{sec:proof_prop_max_policy}

The first part directly follows from Assumption \ref{SPC}.  
Next, we prove the second part. For any $\mu$-optimal policy $\pi$, we define $C^{\pi}=\max_{s\in S}\frac{\theta_{\pi}(s)}{\mu(s)}$.
We first prove that $C^{\pi}$ is finite  if $\pi$ is a $\mu$-optimal policy.
 Since $\mu$ is fixed, we have $\mu(s)>0$ for any $s\in S_0$.  Also we have $\theta_{\pi}(s)$ is upper bounded by $1$, so the ratio $\frac{\theta_{\pi}(s)}{\mu(s)}$ is uniformly bounded for these $s\in S_0$ by $1/\min_{s\in S_0}\mu(s)$. 
 
 On the other hand, for any $s\notin S_0$, by Assumption \ref{SPC}, we have $\theta_{\pi^*}(s)=0$ for the  max-$\mu$-optimal policy $\pi^*$. This would further imply that for all $t\geq 0$, $\mathbb{P}_{\pi^*}(s_t=s;\mu)=0$ by definition of $\theta_{\pi^*}$. Note that $\pi^*$, as a max-$\mu$-optimal policy,  has no-fewer optimal actions covered than any $\mu$-optimal policy $\pi$ for  $s\in S_0$, and thus when starting from $\mu$, $\mathbb{P}_{\pi}(s_t=s;\mu)=0$ for all   $t\geq 0$, i.e., $\pi$ cannot visit states that even $\pi^*$ cannot visit. Hence, we have $\theta_{\pi}(s)=0$, and 
by the convention of   $0/0=0$, we have $C^\pi$ is uniformly bounded for any $\mu$-optimal policy $\pi$. 
 Thus, choosing  $C_{\max} = \sup_{\pi:\mu-optimal}C^{\pi}$ completes the proof. \hfill$\square$

\subsection{Proof of Lemma \ref{pop_pg}}

The proof is given by the following lemma and the fact that $\Delta_{\Dc}(w_\Dc) = 0$. 

\begin{Lemma}\label{generalization}
Suppose Assumptions \ref{realizability}, \ref{v_realizable} hold.
With probability at least $1-\delta$, we have 
\begin{align*}
|\Delta(w)-\Delta_{\Dc}(w)|\le \frac{{4\sqrt{2}B_W}\sqrt{\log (|V||W|/\delta)}}{(1-\gamma)\sqrt{n}}
\end{align*}
for any $w\in W$.
\end{Lemma}
\proof{Proof.} 
By definition, we have for any $w\in W$:
\begin{align}
| \Delta(w)  - \Delta_{\Dc}(w) | &\leq | \ell(w) - \ell_{\Dc}(w)| + \Big|\min_{w\in W}\ell_{\Dc}(w) - \min_{w\in W}\ell(w)\Big|\notag\\
&=\Big| \max_{v\in V}\ell(w,v) - \max_{v\in V}\ell_{\Dc}(w,v)\Big| + \Big|\min_{w\in W}\max_{v\in V}\ell_{\Dc}(w,v) - \min_{w\in W}\max_{v\in V}\ell(w,v)\Big|\notag\\
&\leq \max_{v\in V}\Big| \ell(w,v) - \ell_{\Dc}(w,v)\Big| + \max_{w\in W}\max_{v\in V}\Big|\ell_{\Dc}(w,v) - \ell(w,v)\Big|.
\end{align}
First, we can bound each term using Hoeffding's inequality: for each $w, v$, we have
\begin{align}
| \ell(w,v) - \ell_{\Dc}(w,v) | \leq  | (u_{\Dc} - u)^\top w| + |v^\top(K - K_{\Dc}) w| + (1 - \gamma) |v^\top (\mu_{\Dc} - \mu)|.
\end{align}
Now, with probability at least $1 - \delta/(|V||W|)$, we have:  
\begin{align}
 | (u_{\Dc} - u)^\top w| &\leq \sqrt{2} B_W\frac{\sqrt{\log (|V||W|/\delta)}}{\sqrt{n}} \nonumber \\
|v^\top(K - K_{\Dc}) w| &\leq 2  \sqrt{2} B_W \frac{\sqrt{\log (|V||W|/\delta)}}{(1 - \gamma)\sqrt{n}} \nonumber \\
 (1 - \gamma) |v^\top (\mu_{\Dc} - \mu)| &\leq  \sqrt{2} \frac{\sqrt{\log (|V||W|/\delta)}}{\sqrt{n}}.
\end{align} 
Finally, taking a union bound over all $w \in W$ and $v \in V$, and noting that $B_W\geq 1$ since $w^*\in W$ and $B_W\geq \|w^*\|_\infty\geq 1$,  we get the desired result.
\hfill \moronly{$\square$}
\endproof

\subsection{Proof of Lemma \ref{inactive}}\label{proof:inactive}
We have
\begin{align}
&\Delta(w_{\Dc}) = \ell(w_{\Dc})-\ell(w^*) \nonumber \\
&\quad\ge \ell(w_{\Dc},v^*)-\ell(w^*,v^*) \nonumber \\
&\quad= \left( -\sum_{s,a}r(s,a)\mu(s,a)w_{\Dc}(s,a)+\sum_{s,a}w_{\Dc}(s,a)\mu(s,a) \left( v^*(s)-\gamma \sum_{s'\in S}P_{s,a}(s')v^*(s') \right) - (1-\gamma)v^{*^\top}\mu \right)  \nonumber \\
& \qquad\quad - \left( -\sum_{s,a}r(s,a)\mu(s,a)w^*(s,a)+\sum_{s,a}w^*(s,a)\mu(s,a) \left( v^*(s)-\gamma \sum_{s'\in S}P_{s,a}(s')v^*(s') \right) - (1-\gamma)v^{*^\top} \mu \right) \nonumber\\
&\quad=\sum_{s,a} \left( w_{\Dc}(s,a)\mu(s,a)-w^*(s,a)\mu(s,a) \right) \left( v^*(s)- \left( r(s,a)+\gamma \sum_{s'}P_{s,a}(s')v^*(s') \right) \right) \nonumber \\
&\quad=\sum_{s,a} \left( w_{\Dc}(s,a)\mu(s,a)-w^*(s,a)\mu(s,a) \right) \left( v^*(s)-Q^*(s,a) \right) \nonumber \\
&\quad\ge \Delta_Q\sum_{(s,a)\in \mathcal{I}}w_{\Dc}(s,a)\mu(s,a),
\end{align}
where the first inequality is due to the definition of $\ell(\cdot)$, the second to forth equalities are due to the definitions of $\ell(\cdot)$ and $Q^*$.
The second inequality uses the fact that $w^*(s,a)\mu(s,a)=\theta_{\pi^*}(s,a)=0$ for $(s,a)\in \mathcal{I}$ (see Lemma \ref{comple})
and the definition of $\Delta_Q$. This completes the proof. \hfill $\square$

\subsection{Proof of Lemma \ref{J_Delta}}\label{proof:J_Delta}


%
%

Recall $\theta_{\Dc}(s,a)=w_{\Dc}(s,a)\mu(s,a)$. 
Next, we define a policy $\tilde{\pi}^*$ as follows:
\begin{enumerate}
\item For any $s\in S_0$, we define
$$
\mathcal{T}_\mu(s):=\left\{a \bigm| a\in \mathcal{T}(s) ~~\text{and}~~\mu(s,\tilde{a})>0\right\}.
$$ 
By Assumption \ref{SPC}, we know that such a set $\mathcal{T}_\mu(s)$ is non-empty.  
Then, for any such $\tilde{a}\in \mathcal{T}_\mu(s)$,  
we let  $\tilde{\theta}(s,\tilde{a})=\theta_{\Dc}(s,\tilde{a})+\frac{1}{|\mathcal{T}_\mu(s)|}\sum_{a': (s,a')\in \mathcal{I}}\theta_{\Dc}(s,a')$. 
For any other $a'\notin \mathcal{T}_\mu(s)$ but $a'\in \mathcal{T}(s)$, 
we let $\tilde{\theta}(s,a')=\theta_{\Dc}(s,a')$. Note that these cover all the $a$ such that $(s,a)\in S\times A\setminus \mathcal{I}$. 
Finally, for any other $a'$ 
with $(s,a')\in \mathcal{I}$, we let $\tilde{\theta}(s,a')=0$.

Then for $s\in S_0$, $\tilde\pi^*(a\mid s)$ is generated  by normalizing $\tilde{\theta}(s,\cdot)$, i.e., $\tilde\pi^*(a\mid s)=\tilde{\theta}(s,a)/\sum_{a'\in A}\tilde{\theta}(s,a')$. {Note that by definition we have $\sum_{a'\in A}\tilde{\theta}(s,a')=\sum_{a'\in A}{\theta}_\Dc(s,a')$, and by the lower bound constraint in \eqref{equ:def_W}, we have 
\begin{align*}
	\sum_{a'\in A}{\theta}_\Dc(s,a')=\sum_{a'\in A}w_{\Dc}(s,a')\pi_{\mu}(a'\mid s)\mu(s)\geq (1-\gamma)\mu(s)>0.
\end{align*}
Hence, the normalization of obtaining $\tilde\pi^*(a\mid s)$ is not degenerate in the sense that $\sum_{a'\in A}\tilde{\theta}(s,a')>0$.}  
\item For any $s\notin S_0$, we choose any $\hat{a}$ 
such that $(s,\hat{a})\in S\times A\setminus \mathcal{I}$, i.e., $\hat{a}$ that maximizes $Q^*(s,a)$.
Then we let $\tilde\pi^*(\hat{a} \mid s)=1$ and set $\tilde\pi^*(a'\mid s)=0$ for any other $a'\in A$ and $a'\neq  \hat{a}$.   
\end{enumerate}

Note that this $\tilde\pi^*$ is an optimal policy by construction and by Lemma \ref{comple}. 
Moreover, the next lemma shows that $\tilde{\pi}^*$ is a $\mu$-optimal policy. 

\begin{Lemma}\label{coverage_pistar}
Under  Assumption \ref{SPC}, 
$\tilde{\pi}^*$ is a $\mu$-optimal policy. {Furthermore, 
 $\theta_{\tilde{\pi}^*}(s, \cdot)=0$ for any $s\notin S_0$}. 
\end{Lemma}
\proof{Proof.}
By the construction of $\tilde{\pi}^*$, we know that for any $s\in S_0$, we have $\mu(s,a)>0$ if $\tilde{\pi}^*(a\mid s)>0$. 
Also we know that $\tilde{\pi}^*$ is an optimal policy due to Lemma \ref{comple}. Then $\tilde{\pi}^*$ is a $\mu$-optimal policy.
The second part follows directly from Assumption \ref{SPC}.
Specifically, first, Assumption \ref{SPC} implies that for any $s\notin S_0$, since $\mu(s,a)=0$ for all $a\in A$, we have $\theta_{\pi^*}(s,a)=0$ for all $a\in A$. Then, by the definition of max-$\mu$-optimal policy, the states visited by the $\mu$-optimal policy $\tilde{\pi}^*$ should be visited  by the max-$\mu$-optimal policy $\pi^*$, when starting from $\mu$. This is because the max-$\mu$-policy $\pi^*$ can take all actions that $\tilde{\pi}^*$ can take (see also the argument in \S\ref{sec:proof_prop_max_policy}). Hence, whenever $\theta_{\pi^*}(s)=0$, we should also have 
$\theta_{\tilde{\pi}^*}(s)=0$ and thus $\theta_{\tilde{\pi}^*}(s,\cdot)=0$. This completes the proof. 
\hfill \moronly{$\square$}
\endproof

Now, we prove Lemma \ref{J_Delta} using Lemmas \ref{inactive}, \ref{coverage_pistar}, Proposition \ref{max_policy}, and the performance difference lemma \citep{kakade2002approximately}.

\proof{Proof of Lemma \ref{J_Delta}.} 
We use performance difference lemma to $\tilde{\pi}^*,\pi_{\Dc}$ 
to obtain
\begin{align*}
J_{\mu}(\tilde{\pi}^*)-J_{\mu}(\pi_{\Dc}) \le \frac{1}{1-\gamma}\sum_{s\in S}\theta_{\tilde{\pi}^*}(s)\|\pi_{\Dc}(\cdot\mid s)-\tilde{\pi}^*(\cdot\mid s)\|_1.	
\end{align*}
Because $\theta_{\tilde{\pi}^*}(s,a)=0$ for $s\notin S_0$ by Lemma \ref{coverage_pistar}, we have 
\begin{align*}
J_{\mu}(\tilde{\pi}^*)-J_{\mu}(\pi_{\Dc}) \le \frac{1}{1-\gamma}\sum_{s\in S_0}\theta_{\tilde{\pi}^*}(s)\|\pi_{\Dc}(\cdot\mid s)-\tilde{\pi}^*(\cdot\mid s)\|_1.	
\end{align*}
By the construction of $\tilde{\pi}^*$, we have
\begin{align}\label{key+}
J_{\mu}(\tilde{\pi}^*)-J_{\mu}(\pi_{\Dc})&\le \frac{1}{1-\gamma}\sum_{s\in S_0}\theta_{\tilde{\pi}^*}(s)\|\pi_{\Dc}(\cdot\mid s)-\tilde{\pi}^*(\cdot\mid s)\|_1\\
&=^{(*1)} \frac{2}{1-\gamma}\sum_{s\in S_0}\frac{\theta_{\tilde{\pi}^*}(s)}{\sum_{a}\theta_{\Dc}(s,a)}\cdot \sum_{a:(s,a)\in \mathcal{I}}\theta_{\Dc}(s,a) \\ 
&=\frac{2}{1-\gamma}\sum_{s\in S_0}\frac{\theta_{\tilde{\pi}^*}(s)/\mu(s)}{\sum_{a}\theta_{\Dc}(s,a)/\mu(s)}\cdot \sum_{a:(s,a)\in \mathcal{I}}\theta_{\Dc}(s,a)\\
&\le^{(*2)}\frac{2C_{\max}}{(1-\gamma)^2}\sum_{(s,a)\in \mathcal{I}}\theta_{\Dc}(s,a) \label{eq:used_in_6}\\ 
&\le^{(*3)} \frac{2C_{\max}}{(1-\gamma)^2\Delta_Q}\cdot \Delta(w_{\Dc}),
\end{align}  
where $(*1)$ follows from the construction of $\tilde{\pi}^*$ from $\theta_{\Dc}$,
$(*2)$ is because of Proposition \ref{max_policy} 
and definition of $S_0$, and $(*3)$ is due to Lemma \ref{inactive}. 
\hfill \moronly{$\square$}
\endproof


%
%



\bibliographystyle{plainnat} 
\bibliography{refs}



\end{document}